\pdfoutput=1
\documentclass{llncs}

\usepackage{times}
\usepackage{helvet}
\usepackage{courier}

\usepackage{stmaryrd}
\usepackage{url}
\usepackage{amsfonts}
\usepackage{amsmath}
\usepackage{comment}
\usepackage{xspace}
\usepackage{enumerate}
\usepackage{paralist}
\usepackage{todonotes}
\usepackage{algorithm}
\usepackage{algorithmic}

\makeatletter
\renewcommand*{\ALG@name}{Procedure}
\makeatother

\newcommand{\Math}[1]{\ensuremath{#1}}

\newcommand{\modecal}[1]{{\Math{\mathcal{#1}}}}
\newcommand{\modeit}[1]{{\Math{\mathit{#1}}}}

\newcommand{\textmath}[1]{\mbox{\textit{#1}}}

\newcommand{\propername}[1]{\mbox{\small \textsf{#1}}}

\newcommand{\Golog}{\propername{Golog}}

\newcommand{\ConGolog}{\propername{ConGolog}}

\newcommand{\adom}{adom}
\newcommand{\const}{const}

\newcommand{\free}{free}

\newcommand{\Lra}{\Leftrightarrow}
\newcommand{\ra}{\rightarrow}

\newcommand{\A}{\modecal{A}} 
\newcommand{\C}{\modecal{C}} \newcommand{\D}{\modecal{D}}
\newcommand{\E}{\modecal{E}} \newcommand{\F}{\modecal{F}}
 
\newcommand{\I}{\modecal{I}} 
 \renewcommand{\L}{\modecal{L}}
\newcommand{\M}{\modecal{M}} \newcommand{\N}{\modecal{N}}
 
\renewcommand{\S}{\modecal{S}}

\newcommand{\R}{\modecal{R}} 

\newcommand{\Mod}{\mathit{Mod}}

\newcommand{\Poss}{\textmath{Poss}}

\def\planeaux!#1:#2<-#3!{\Math{#1 \mbox{\rm:} #2\; \leftarrow #3}}

\def\planeaux!#1<-#2!{\Math{#1 \leftarrow #2}}

\newcommand{\set}[1]{\{#1\}}                      

\newcommand{\card}[1]{|{#1}|}                     

\newcommand{\tuple}[1]{\Math{\langle #1 \rangle}}		
\newcommand{\tup}[1]{\tuple{#1}}            			

\newcommand{\myi}{\emph{(i)}\xspace}
\newcommand{\myii}{\emph{(ii)}\xspace}
\newcommand{\myiii}{\emph{(iii)}\xspace}

\newcommand{\true}{\mathtt{true}}
\newcommand{\false}{\mathtt{false}}

\long\def\eatpar#1{%
\ifx#1\par                      
\let\nextmove=\eatpar           
\else
\let\nextmove=#1
\fi
\nextmove
}

\def\qed{\hfill{\qedboxempty}      
  \ifdim\lastskip<\medskipamount \removelastskip\penalty55\medskip\fi}

\def\qedboxempty{\vbox{\hrule\hbox{\vrule\kern3pt
                 \vbox{\kern3pt\kern3pt}\kern3pt\vrule}\hrule}}

\def\qedfull{\hfill{\qedboxfull}   
  \ifdim\lastskip<\medskipamount \removelastskip\penalty55\medskip\fi}

\def\qedboxfull{\vrule height 4pt width 4pt depth 0pt}

\newcounter{bean}

\newenvironment{tightenumerate}{
                \begin{list}{
                  {\mbox {
                      \arabic{bean}.\/}}}{\usecounter{bean}
                      \setlength{\itemsep}{-1pt}\setlength{\topsep}{0pt}}}{
                \end{list}}

\newenvironment{tightitemize}{
                \begin{list}{$\bullet$}{
                    \setlength{\itemsep}{-1pt}}{\setlength{\topsep}{0pt}}}{
                \end{list}}

\newcommand{\page}{pp.}

\newcommand{\under}[1]{\mbox{\underline{\it\smash{#1}\vphantom{\lower.05ex\hbox{
x}}}}}

\newcommand{\commentarea}[1]{}

\newcommand{\Executable}{\modeit{Executable}}
\newcommand{\refresh}{\modeit{refresh}}
\newcommand{\forget}{\modeit{forget}}

\newcommand{\BOX}[1]{ [-] #1}
\newcommand{\DIAM}[1]{\langle - \rangle #1}

\newcommand{\vfo}{\ensuremath{v}}
\newcommand{\vso}{\ensuremath{V}}
\newcommand{\MODA}[1]{(#1)_{(\vfo,\vso)}^{\M}}   
\newcommand{\MODAX}[2]{(#1)_{(\vfo ,\vso)#2}^{\M}}   
\newcommand{\MODAZ}[2]{(#1)_{(\vfo,\vso)#2}^{\M}}  
\newcommand{\MODAT}[1]{(#1)_{(\vfo,\vso)}^{T}}   
\newcommand{\MODATM}[1]{(#1)_{(\vfo,\vso)}^{T_\M}}   
\newcommand{\MODATX}[2]{(#1)_{(\vfo,\vso)#2}^{T}}   
\newcommand{\MODATZ}[2]{(#1)_{(\vfo,\vso)#2}^{T}}  

\newcommand{\limp}{\supset}

\newcommand{\NextBounded}{\mathit{NextOrigBounded}}

\newcommand{\inadom}{\ensuremath{\textsc{live}}}
\newcommand{\mymu}{\ensuremath{\mu\L_p}\xspace}
\newcommand{\myhm}{\ensuremath{\L_{P}}\xspace}

\newcommand{\domain}[1]{\ensuremath{\textsc{dom}(#1)}\xspace}
\newcommand{\image}[1]{\ensuremath{\textsc{im}(#1)}\xspace}
\newcommand{\restrict}[2]{\ensuremath{#1 \vert_{#2}}}

\newcommand{\arestrict}[1]{\ensuremath{\tilde{#1}}}

\newcommand{\LET}{\textbf{let}\xspace}

\begin{document}

\title{Bounded Situation Calculus Action Theories}

\author{
	Giuseppe De Giacomo\inst{1}
	\and
	Yves Lesp\'{e}rance\inst{2}
	\and
	Fabio Patrizi\inst{3}
}

\institute{
	Dipartimento di Ingegneria informatica, automatica e gestionale\\
	Sapienza Universit\`a di Roma, Italy\\
	\email{degiacomo@dis.uniroma1.it}
	\and
	Department of Electrical Engineering and Computer Science\\
	York University, Toronto, ON, Canada\\
	\email{lesperan@cse.yorku.ca}
	\and
	KRDB Research Centre -- Faculty of Computer Science\\
	Free University of Bozen-Bolzano\\
	\email{patrizi@dis.uniroma1.it}
}


\maketitle
\begin{abstract}

In this paper,\footnote{A preliminary version of this paper appeared as
\cite{DBLP:conf/kr/GiacomoLP12}.} we investigate bounded action theories in the situation
calculus. A bounded action theory is one which 
entails that, in every situation, the number of object tuples in the extension of 
fluents is bounded by a given constant, although such extensions are in general different 
across the infinitely many situations. 
We argue that such theories are common in
applications, either because facts do not persist indefinitely or
because the agent eventually forgets some facts, as new ones are learnt. 
We discuss various
classes of bounded action theories. 
Then we show that verification of 
a powerful first-order variant of the $\mu$-calculus
is decidable for such theories.
Notably, this variant supports
a controlled form of quantification across situations.
We also show that through verification, we can actually check whether 
an arbitrary action theory maintains boundedness.
\end{abstract}


\section{Introduction}\label{sec:intro}

The situation calculus \cite{McCarthy1969:AI,Reiter01-Book} is a
well-known first-order formalism with certain second-order features
for representing dynamically changing worlds. It has proved to be an
invaluable formal tool for understanding the subtle issues involved in
reasoning about action.  Its comprehensiveness allows us to place all
aspects of dynamic systems in perspective.  Basic action theories let
us capture change as a result of actions in the system
\cite{Reiter91Frame}, while high-level languages such as \Golog~
\cite{Levesque:JLP97-Golog} and
\ConGolog~\cite{DeGiacomoLL:AIJ00-ConGolog} support the representation
of processes over the dynamic system. Aspects such as time
\cite{Reiter96}, knowledge and sensing
\cite{DBLP:journals/ai/ScherlL03}, probabilities and utilities
\cite{DBLP:conf/aaai/BoutilierRST00}, and preferences
\cite{DBLP:conf/kr/BienvenuFM06}, have all been addressed.

The price of such a generality is that decidability results for
reasoning in the situation calculus are rare, e.g.,
\cite{DBLP:conf/ijcai/Ternovskaia99} for an argument-less fluents
fragment, and \cite{DBLP:conf/ijcai/GuS07} for a description
logic-like two-variable fragment.
Obviously, we have the major feature of being able to rely on
regression to reduce reasoning about a given future situation to
reasoning about the initial situation
\cite{Reiter01-Book}. Generalizations of this basic result such as
just-in-time histories \cite{DBLP:conf/ijcai/GiacomoL99} can also be
exploited.
However, when we move to temporal properties, virtually all approaches
are based on assuming a finite domain and a finite number of states,
and often rely on propositional modal logics and model checking
techniques \cite{BaKG08,LomuscioQR09}. There are only few exceptions
such as \cite{Classen:KR08,DBLP:conf/kr/GiacomoLP10,ShLL10}, which
develop incomplete fixpoint approximation-based methods.

In this paper, we present an important new result on decidability of
the situation calculus, showing that \emph{verification of bounded
  action theories is decidable}.
Bounded action theories are basic action theories
\cite{Reiter01-Book}, where it is entailed that in all situations, the
number of object tuples that belong to the extension of any fluent is
bounded.  In such theories, the object domain remains nonetheless
infinite
and an infinite run may involve an infinite number of objects, though
at every single situation the number of objects we predicate on is
finite and, in fact, bounded.

But why should we believe that practical domains conform to this
boundedness assumption?  While it is often assumed that the law of
inertia applies and that fluent atoms persist indefinitely in
the absence of actions that affect them, we all know that pretty much
everything eventually decays and changes.  We may not even know how
the change may happen, but nevertheless know that it will.
Another line of argument for boundedness is epistemic.  Agents
remember facts that they use and periodically try to confirm them,
often by sensing.  A fact that never gets used is eventually
forgotten.  If a fact can never be confirmed, it may be given up as
too uncertain.  Given this, it seems plausible that in several contexts an agent's
knowledge, in every single moment, can be assumed to be bounded.
While these philosophical arguments are interesting and relate to some
deep questions about knowledge representation, one may take a more
pragmatic stance, and this is what we do here.  We identify some
interesting classes of bounded action theories and show how they can
model typical example domains.  We also show how we can transform
arbitrary basic action theories into bounded action theories, either
by blocking actions that would exceed the bound, or by having
persistence (frame axioms) apply only for a finite number of steps.
Moreover we show that we can effectively check whether any arbitrary
theory with a bounded initial situation description remains bounded in
all executable situations (to do so we need to use verification).

The main result of the paper is that verification of an expressive
class of first-order $\mu$-calculus temporal properties in bounded
action theories is decidable and in fact EXPTIME-complete.  This means that
we can check whether a system or process specified over such a theory
satisfies some specification even if we have an infinite domain and an
infinite set of situations or states.
In a nutshell, we prove our results by focussing on the \emph{active
  domain} of situations, i.e., the set of objects for which some
atomic fluent holds; we know that the set of such active objects is
bounded.  We show that essentially we can abstract situations whose
active domains are \emph{isomorphic} into a single state, and thus, by
suitably abstracting also actions, we can obtain an \emph{abstract
  finite transition system} that \emph{satisfies exactly the same
  formulas} of our variant of the $\mu$-calculus.

This work is of interest not only for AI, but also for other areas of
computer science. In particular it is of great interest for the work
on data-aware business processes and services
\cite{Hull2008:Artifact,GeredeSu:ICSOC2007,Dumas2005:PAIS}. Indeed
while there are well-established results and tools to analyze business processes
and services, without considering the data manipulated, when data are taken into account results are scarce. The present work complements that in, e.g.,
\cite{DHPV:ICDT:09,BPM11,ICSOC11,DBLP:conf/pods/HaririCGDM13,Belardinelli.etal:JAIR14}, and hints at an even more profund relevance of the situation calculus in those areas \cite{MaMS14}. More generally, our results can be recast in other reasoning about action formalisms, both in AI and in CS.

The rest of the paper is organized as follows. In
Section~\ref{sec:preliminaries}, we briefly review the situation
calculus and basic action theories. Then in Section~\ref{sec:bounded},
we define bounded action theories.  Following that in
Section~\ref{sec:obtaining}, we discuss various ways of obtaining
bounded action theories, while showing that many practical domains can
be handled.  In Section~\ref{sec:mucalc}, we introduce the \mymu
language that we use to express first-order temporal properties and
its semantics.  After that, we show that verification of \mymu
properties over bounded action theories is decidable, first in the
case where we have complete information about the initial situation in
Section~\ref{sec:verification}, and then in the general incomplete
information case in Section~\ref{sec:incompleteInfo}.  Then in
Section~\ref{sec:complexity}, we characterize the worst-case
computational complexity of the problem as EXPTIME-complete.  In
Section~\ref{sec:boundedness}, we give a technique based on our
verification results to check whether an arbitrary basic action theory
is maintains boundedness. In Section~\ref{sec:rel-work}, we review the
related literature. Finally, in Section~\ref{sec:conclusion}, we
conclude the paper mentioning topics for future work.


\section{Preliminaries}\label{sec:preliminaries}


The \textit{situation calculus} \cite{McCarthy1969:AI,Reiter01-Book}
is a sorted predicate logic language for representing and reasoning
about dynamically changing worlds.
All changes to the world are the result of \textit{actions},
which are terms in the language. We denote action variables by
lower case letters $a$, action types by capital letters $A$, and
action terms by $\alpha$, possibly with subscripts.
A possible world history is represented by a term called a
\emph{situation}. The constant $S_0$ is used to denote the initial
situation where no actions have yet been performed. Sequences of
actions are built using the function symbol $do$, where $do(a,s)$
denotes the successor situation resulting from performing action $a$
in situation $s$. 
Besides actions and situations, there is also the sort of \textit{objects} for
all other entities.
Predicates and functions whose value varies from situation to situation are
called \textit{fluents}, and are denoted by symbols taking a
situation term as their last argument (e.g., $Holding(x,s)$, meaning
that the robot is holding object $x$ in situation $s$).
For simplicity, and without loss of generality, we assume that there are no functions
other than constants and no non-fluent predicates.  We denote fluents
by $F$ and the finite set of primitive fluents by $\F$.  The arguments
of fluents (apart from the last argument which is of sort situation) are
assumed to be of sort object.

Within this language, one can formulate action theories that describe
how the world changes as the result of the available actions.  Here,
we concentrate on \emph{basic action theories} 
as proposed in \cite{PirriR:JACM99,Reiter01-Book}.  We also assume
that there is a \emph{finite number of action types}.
Moreover, we assume that there is a countably
infinite set of object constants $\N$ for which the unique name
assumption holds.  But we do not assume domain closure for objects.\footnote{
Such an assumption is made in \cite{DBLP:conf/kr/GiacomoLP12}, where \emph{standard names} \cite{LeLa01} are used to denote objects. Thus, the results here generalize those in \cite{DBLP:conf/kr/GiacomoLP12}.}
As a result a basic action theory $\D$ is the union of the
following disjoint sets of first-order (FO) and second-order (SO) axioms: 
\begin{itemize}
	\item $\D_{0}$: (FO) \emph{initial situation description axioms} 
		describing the initial configuration of the world (such a description may be complete or incomplete);
              \item $\D_{poss}$: (FO) \emph{precondition axioms} of
                the form \[\Poss(A(\vec x), s)\equiv \phi_A(\vec x,s),\]
                one per action type, stating the conditions
                $\phi_A(\vec x,s)$ under which an action $A(\vec x)$
                can be legally performed in situation $s$; these use a
                special predicate $\Poss(a,s)$ meaning that action $a$
                is executable in situation $s$; $\phi_A(\vec x,s)$ is
                a formula of the situation calculus that is
                \emph{uniform} in situation $s$, that is, a formula
                that mentions no other situation term but $s$ and does
                not mention $\Poss$  (see \cite{Reiter01-Book} for a
                formal definition);

              \item $\D_{ssa}$: (FO) \emph{successor state axioms} of
                the form
                \[F(\vec x,do(a,s))\equiv\phi_F(\vec x,a,s),\]
                one per fluent, describing how the fluent changes when
                an action is performed; the right-hand side (RHS)
                $\phi_F(\vec x,a,s)$ is again a situation calculus
                formula uniform in $s$; successor state axioms encode
                the causal laws of the world being modeled; they take
                the place of the so-called effect axioms and provide a
                solution to the frame problem;

	\item $\D_{ca}$: (FO) unique name axioms for actions and (FO) domain closure on action types; 

	\item $\D_{uno}$: (FO) unique name axioms for object constants in $\N$; 

	\item  $\Sigma$: (SO) foundational, domain independent, axioms
of the situation calculus~\cite{PirriR:JACM99}.
\end{itemize}
We say that a situation $s$ is \emph{executable}, written
 $\Executable(s)$, if every action performed in
reaching $s$ was executable in the situation in which it occurred.

One of the key features of basic action theories is the
existence of a sound and complete \textit{regression mechanism} for
answering queries about situations resulting from performing a sequence of actions
\cite{PirriR:JACM99,Reiter01-Book}.
In a nutshell, the regression operator $\R^*$ reduces a formula
$\phi$ about a particular future situation to an equivalent formula
$\R^*[\phi]$ about the initial situation $S_0$, by basically
substituting fluent relations with the right-hand side formula of
their successor state axioms.
Here, we shall use a simple \textit{one-step only} variant $\R$ of the standard regression
operator $\R^*$ for basic action theories.
Let $\phi(do(\alpha,s))$ be a formula uniform in the situation
$do(\alpha,s)$. 
%
Then
$\R[\phi(do(\alpha,s))]$ stands for the \textit{one-step} regression of $\phi$ through the action term
$\alpha$, which is itself a formula uniform in $s$.

\section{Bounded Action Theories}\label{sec:bounded}

Let $b$ be some natural number. We use the notation $\card{\{\vec{x} \mid \phi(\vec{x})\}} \ge b$, meaning 
that there exist at least $b$ distinct tuples that satisfy $\phi$, to stand for the following FOL formula:
\[\exists \vec{x}_1,\ldots,\vec{x}_b. \phi(\vec{x}_1)\land \cdots\land \phi(\vec{x}_b) \land \bigwedge_{ i,j \in\{1,\ldots,b\},  i\neq j}\vec{x}_i\neq \vec{x}_j.\]
We also use the notation $\card{\{\vec{x} \mid \phi(\vec{x})\}}  < b$,
meaning that there are fewer than $b$ distinct tuples that satisfy $\phi$,
to stand for: $\lnot (\card{\{\vec{x} \mid \phi(\vec{x})\}} \ge b)$. 

Using this, we define the notion of a fluent $F(\vec{x},s)$ in situation $s$ being \textit{bounded} by a natural number $b$ as follows:
\[Bounded_{F,b}(s) \doteq \card{\{\vec{x} \mid F(\vec{x},s)\}}  < b,\]
i.e., fluent $F$ is bounded by $b$ in situation $s$ 
if there are fewer than $b$ distinct tuples in the extension 
of $F$ in situation $s$.

The notion of situation $s$  being bounded by a natural number $b$ is
defined as follows:
\[Bounded_{b}(s) \doteq \bigwedge_{F\in \F} Bounded_{F,b}(s),\]
i.e., every fluent is bounded by $b$ in situation $s$.

We say that an action theory $\D$ is \emph{bounded}  by  $b$ if 
every executable situation is bounded by $b$, formally: 
\[\D\models \forall s. \Executable(s) \limp Bounded_b(s).\]

\begin{example}
  Consider a warehouse where items are moved around by a robot (a
  similar example is formalized in \cite{DeGiacomoLPVAAMAS14}).  There
  are $k$ storage locations where items can be stored.  There is also
  a shipping dock where new items may arrive and stored items may be
  shipped out.  We can axiomatize this domain as follows.

We have the following action precondition axioms:\footnote{Throughout
  this paper, we assume that all free variables in a formula are
  implicitly universally quantified from the outside. Occasionally, to be clear, we will write $\forall\varphi$ to denote the universal closure of $\varphi$ explicitly.}
\[\begin{array}{l}
\Poss(move(x,l,l'), s) \equiv At(x,l,s) \land IsLoc(l') \land \lnot \exists y At(y,l',s)

\\[1ex]
\Poss(arrive(x), s) \equiv \lnot \exists y At(y,ShipDock) \land \lnot
\exists l At(x,l,s) 
\\[1ex]
\Poss(ship(x),s) \equiv At(x,ShipDock,s)
  \end{array}\]
  The first axiom says that in situation $s$, the robot can perform
  action $move(x,l,l')$, i.e., move object $x$ from location $l$ to
  $l'$, if and only if $x$ is at $l$ in $s$ and $l'$ is a location
  where no object is present in $s$.  The second precondition axiom says that
  action $arrive(x)$ is executable in situation $s$, i.e., 
  object $x$ may arrive at the warehouse in $s$, if and only if the
  shipping dock is empty and $x$ is not somewhere else in the warehouse.  The last axiom says that
  object $x$ can be shipped in situation $s$ if it is at the shipping
  dock in $s$.

  For the fluent $At$, we have the following successor state axiom:
\[\begin{array}{l}
At(x,l,do(a,s)) \equiv \gamma(x,l,a,s)^+ \lor At(x,l,s) \land \lnot\gamma^-(x,l,a,s),\\
\mbox{where}\\
\quad				\gamma^+(x,l,a,s) = \exists r. a=move(x,l',l) \land 
				At(x,l',s) \land IsLoc(l') \land \lnot \exists y
                                At(y,l,s) \\
\hspace{7.5em}            {} \lor                   a = arrive(x) \land l = ShipDock\ \mbox{ and}\\[1ex]

\quad				\gamma^-(x,l,a,s) =
                                \exists l'. a=move(x,l,l') \land l'\neq
				l \land IsLoc(l') \land \lnot \exists y At(y,l',s)\\
\hspace{7.5em}  {}                                \lor
                                a=ship(x) \land At(x,ShipDoc,s)
  \end{array}\]
  This says that object $x$ is at location $l$ in the situation that
  results from doing action $a$ in $s$ if and only if 
  $\gamma(x,l,a,s)^+$ holds or
  if $x$ is already at $l$ in $s$ and $\gamma^-(x,l,a,s)$ doesn't
  hold.  $\gamma(x,l,a,s)^+$ specifies the conditions under which
  action $a$ makes object $x$ be at location $l$ in situation $s$,
  i.e.,  if $a$ is to move $x$ to a free location $l$ from another
  location $l'$ where $x$ was in $s$, or $a$ is $x$ arriving and $l$
  is the shipping dock.  $\gamma^-(x,l,a,s)$ specifies the conditions
  under which action $a$ makes object $x$ cease to be at location $l$
  in situation $s$, i.e., $a$ is to move $x$ to a different location
  that is free, or is to ship $x$.

  We specify the initial situation with the following initial state
  axioms:
\[\begin{array}{l}
\neg At(x,l,S_0)
\\[1ex]
IsLoc(l) \equiv l = ShipDock \lor l = SL_1 \lor \ldots \lor  l = SL_k
\end{array}\]
We also have unique name axioms for the locations.
For clarity, we make $IsLoc$ a non-fluent predicate, although it
is easy to recast it as a fluent that is unaffected by any action.

It is not difficult to show that this theory is in fact bounded by
$k+1$.  First note that there are $k+1$ locations initially and the
set of locations never changes, so $IsLoc$ is bounded by $k+1$.  For fluent $At$,
it is initially bounded by $0$, but the $arrive$ action can augment
its extension.  However, the action theory ensures there can be at
most one item at each of the $k+1$ locations.  Thus $At$ remains
bounded by $k+1$.  Therefore, the theory is bounded by $k+1$.

Observe that, as there are infinitely many constants denoting distinct
objects, effectively an unbounded number of items may be handled by
subsequent $arrive$, $move$, and $ship$ actions.  Despite this, the
theory remains bounded.
\end{example}

We shall see  that for bounded action theories, verification of
sophisticated temporal properties is decidable.


\section{Obtaining Bounded Action Theories}\label{sec:obtaining}

Before focusing on verification, in this section we look at various interesting sufficient conditions that guarantee that a basic action theory is bounded. 
Later in Section~\ref{sec:boundedness}, we will see that it is actually possible to use verification itself to check whether any arbitrary basic action theory, with a bounded initial situation description, is indeed bounded. 

\subsection{Bounding by Blocking}

We observe that the formula $Bounded_b(s)$ is a FO formula
uniform in $s$ and hence it is regressable for basic action theories.
This allows us to introduce a first interesting class of bounded
action theories. Indeed, from any basic action theory, we can
immediately obtain a bounded action theory by simply blocking the
execution of actions whenever the result would exceed the bound.

Let $\D$ be a basic action theory. We define the bounded basic action
theory $\D_b$ by replacing each action precondition axiom in $\D$ of
the form
$Poss(a(\vec{x}),s) \equiv \Phi(\vec{x},s)$ 
by a precondition axiom of the form
\begin{equation}\label{boundingpos}
Poss(a(\vec{x}),s) \equiv \Phi(\vec{x},s)\land \R[Bounded_b(do(a(\vec{x}),s))]
\end{equation}

\begin{theorem}
Let $\D$ be a basic action theory with the initial
  description $\D_0$ such that $\D_0\models Bounded_b(S_0)$, for some
  $b$, and let $\D_b$ be the basic action theory obtained as discussed above.
Then, $\D_b$ is bounded by 
$b$.
\end{theorem}
\begin{proof}
By~(\ref{boundingpos}) it is guaranteed that any executable action leads to a bounded situation. Hence by induction on executable situations, we get the thesis.
\end{proof}


%
\begin{example}
Suppose that we have a camera on a smart phone or tablet computer.  We could model
the storage of photos on the device using a fluent $PhotoStored(p,s)$,
meaning that photo $p$ is stored in the device's memory.  Such a
fluent might have the following successor state axiom:
\[\begin{array}{l}
PhotoStored(p,do(a,s)) \equiv
a = takePhoto(p) \\
\hspace{4em} {} \lor PhotoStored(p,s) \land  a \neq deletePhoto(p)
\end{array}\]
We may also assume that action $takePhoto(p)$  is always executable and
that $deletePhoto(p)$ is executable in $s$ if $p$ is stored in $s$:
\[\begin{array}{l}
Poss(takePhoto(p), s) \equiv True\\[1ex]
Poss(deletePhoto(p), s) \equiv PhotoStored(p,s).
\end{array}\]

Now such a device would clearly have a limited capacity for storing photos.
If we assume for simplicity that photos come in only one resolution
and file size, then we can model this by simply applying the
transformation discussed above.  This yields the following
modified precondition axioms:
\[\begin{array}{l}
Poss(takePhoto(p), s) \equiv\\
\hspace{3em} {} \card{\{p' \mid
  PhotoStored(p',s)\}} < b-1\\[1ex]
Poss(deletePhoto(p), s) \equiv
PhotoStored(p,s) \land {}\\
\hspace{3em} {}    \card{\{p' \mid
PhotoStored(p',s)\}} < b + 1.
  \end{array}\]
Note how the condition on on the right hand side of the
first axiom above ensures there are fewer than $b$ photos stored after
the action of taking a photo $p$ occurs.
Clearly, the resulting theory is bounded by $b$ (assuming that the
original theory is bounded by $b$ in $S_0$).
\end{example}

Note that this way of obtaining a bounded action theory is far from
realistic in modeling the actual constraints on the storage of photos.
One could develop a  more accurate model, taking into account
the size of photos, the memory management scheme used, etc.
This would  also yield a bounded action theory, though one whose
boundedness is a consequence of a sophisticated model of memory capacity.

\begin{example}
Let's extend the previous example by supposing that the device also
maintains a contacts directory.  We could model this using a fluent
$InPhoneDir(name,number,$ $photo,s)$, with the following successor state axiom:
\[\begin{array}{l}
InPhoneDir(na,no,p,do(a,s)) \equiv\\
\hspace{1.7em} a = add(na,no,p)
\lor InPhoneDir(na,no,p,s) \land {}\\
\hspace{2.7em} a \neq deleteName(na)
\land a \neq deleteNumber(no)
\end{array}\]
We could then apply our transformation to this new theory to obtain a bounded action
theory, getting precondition axioms such as the following:
\[\begin{array}{l}
Poss(add(na,no,p), s) \equiv 
  PhotoStored(p,s) \land {}\\
\hspace{1.7em} {} \card{\{p' \mid
  PhotoStored(p',s)\}} < b \land {}\\
\hspace{1.7em}  \card{\{ \langle na,no,p \rangle \mid
  InPhoneDir(na,no,p,s)\}} < b-1
\end{array}\]
The resulting theory blocks actions from being performed whenever the
action would result in a number of tuples in some fluent
 exceeding the bound.
\end{example}

We observe that this kind of bounded action theories are really
modeling a capacity constraint on every fluent,\footnote{The bound $b$
  applies to each fluent individually, so the total number of tuples
  in a situation is bounded by $\card{\F} b$.  Instead, one could
  equivalently impose a global capacity bound on the total number of
  tuples for which some fluent holds in a situation.}
which may block actions from being executed.
As a result, an action may be executable in a situation in the original
theory, but not executable in the bounded one. Thus an agent may want
to ``plan'' to find a sequence of actions that would make the action
executable again. 
In general, to avoid dead-ends, one should carefully choose the
original action theory on which the bound is imposed, in particular
there should always be actions that remove tuples from fluents.

\subsection{Effect Bounded Action Theories}

Let's consider another sufficient condition for boundedness. Without loss of generality we can take the general form of successor state axioms to be as follows:

\[F(\vec{x},do(a,s)) \equiv \Phi_F^+(\vec{x},a,s) \lor (F(\vec{x},s)\land \lnot \Phi_F^-(\vec{x},a,s))\]

We say that fluent $F$ is \emph{effect bounded} if:

\[\card{\{\vec{x} \mid \Phi_F^+(\vec{x},a,s)\}} \le \card{\{\vec{x} \mid \Phi_F^-(\vec{x},a,s)\}},\]
i.e., for every action and situation, the number of tuples added to
the fluent is less than or equal to that deleted.

We say that a basic action theory is effect bounded if every fluent $F\in\F$ is effect bounded.

\begin{theorem}
  Let $\D$ be an effect bounded basic action theory with the initial situation 
  description $\D_0$ such that $\D_0\models Bounded_b(S_0)$, for some $b$.
  Then $\D$ is bounded by $b$.
\end{theorem}
\begin{proof}
By induction on  executable situations.
\end{proof}

\begin{example}
  Many axiomatizations of the Blocks World are not effect bounded.
  For instance, suppose that we have fluents $OnTable(x,s)$, i.e.,
  block $x$ is on the table in situation $s$, and $On(x,y,s)$, i.e.,
 block $x$ is on block $y$ in situation $s$, with the
  following successor state axioms:
\[\begin{array}{l}
OnTable(x,do(a,s)) \equiv
 a = moveToTable(x) \\
\hspace{4em} {} \lor OnTable(x,s) \land
\neg \exists y.a = move(x,y)
\\[1ex]
On(x,y,do(a,s)) \equiv
 a = move(x,y) \lor On(x,y,s) \land {}\\
\hspace{0.2em}
\neg \exists z.(z \neq y \land a = move(x,z)) \land  a \neq moveToTable(x)
\end{array}\]
Then, performing the action $moveToTable(B1)$ will result in a net increase
in the number of objects that are on the table (assuming that the
action is executable and that $B1$ is not already on the table).  Thus,
fluent $OnTable$ is not effect bounded in this theory.

However, it is easy to develop an alternative axiomatization of the
Blocks World that is effect bounded.  Suppose that we use only the
fluent $On(x,y,s)$ and the single action $move(x,y)$, where $y$ is
either a block or the table, which is denoted by the constant $Table$.
We can axiomatize the domain dynamics as follows:
\[\begin{array}{l}
On(x,y,do(a,s)) \equiv
 a = move(x,y) \\
\hspace{2em} {}\lor On(x,y,s) \land
\neg \exists z.(z \neq y \land a = move(x,z))
\end{array}\]
That is, $x$ is on $y$ after action $a$ is performed in situation $s$
if and only if $a$
is moving $x$ onto $y$ or $x$ is already on $y$  in situation $s$  and
$a$ does not involve moving $x$ onto an object other than $y$.
We say that $move(x,y)$ is executable in situation $s$
if and only if $x$ is not the table in $s$, $x$ and $y$ are distinct, $x$ is clear and
on something other than $y$ in $s$, and $y$ is clear unless it is the table in $s$:
\[\begin{array}{l}
Poss(move(x,y),s) \equiv 
x \neq Table \land x \neq y \land \neg \exists z.On(z,x,s) \land {}\\
\hspace{2em} \exists z.(z \neq y
\land On(x,z,s)) \land (y = Table \lor \neg \exists z.On(z,y,s))
\end{array}\]
Then it is easy to show that any occurence of $move(x,y)$ in a
situation $s$ where the action is executable, adds $\langle x, y
\rangle$ to $O= \{\langle x', y' \rangle \mid On(x',y',s)\}$
while deleting $\langle x, y'' \rangle$ for some
$y''$ s.t.\ $y'' \neq y$, leaving $\card{O}$ unchanged.  Note that
we must require that $x$ be on something in the action precondition
axiom to get this.
Any action other than $move(x,y)$ leaves $O$ unchanged.  Thus $On$ is
effect bounded.

The precondition that $x$ be on something for $move(x,y)$ to be
executable means that we cannot move a new unknown block onto another
or the table. We must of course impose restrictions on ``moving new
blocks in'' if we want to preserve effect boundedness.  One way to do
this is to add an action $replace(x,y)$, i.e.\ replacing $x$ by
$y$.  We can specify its preconditions as follows:
\[\begin{array}{l}
Poss(replace(x,y),s) \equiv
x \neq Table \land y \neq Table \land x \neq y \land {}\\ 
\hspace{2em} \neg \exists z.On(z,x,s) \land \exists z.On(x,z,s) \land \neg \exists z.On(z,y,s) \land\neg \exists z.On(y,z,s)
\end{array}\]
That is, $replace(x,y)$ is executable in situation $s$ if and only if $x$ and $y$ are not
the table and are distinct, $x$ is clear and on something in $s$, and
$y$ is clear and not on something in $s$.
We can modify  the successor state axiom for $On$ to be:
\[\begin{array}{l}
On(x,y,do(a,s)) \equiv  a = move(x,y) \lor  {}\\
\hspace{2em}\exists z.(a = replace(z,x) \land On(z,y,s))\\
\hspace{2em}{}\lor On(x,y,s) \land
\neg \exists z.(z \neq y \land a = move(x,z)) \land {}\\
\hspace{4em}\neg \exists z.(z \neq y \land a = replace(x,z)),
\end{array}\]
where $On(x,y)$ becomes true if $x$ replaces $z$ and $z$ was on $y$ in
$s$, and $On(x,y)$ becomes false if $z$ replaces $x$ and $x$ was on $y$ in
$s$.  It is straightforward to show that this change leaves $On$ effect bounded.
\end{example}

\begin{example}
For another simple example (perhaps more practical), let's look at how
we could specify the ``favorite web sites'' menu of an internet
application.  We can assume that there is a fixed number of favorite
web sites positions on the menu, say $1$ to $k$.  We can
replace what is at position $n$ on the menu by the URL $u$ by
performing the action $replace(n,u)$.  This can be axiomatized as
follows:
\[\begin{array}{l}
FavoriteSites(n,u,do(a,s)) \equiv a = replace(n,u) \lor {}\\
\hspace{2em} FavoriteSites(n,u,s) \land \neg \exists u'.(u' \neq u \land a = replace(n,u'))
\end{array}\]
\[\begin{array}{l}
Poss(replace(n,u),s) \equiv
n \in [1..k] 
\land
\exists u'.FavoriteSites(n,u',s)
\end{array}\]
It is easy to show that in this axiomatization, $FavoriteSites$ is
effect bounded.  No action, including $replace(n,u)$, causes the
extension of the fluent to increase.
\end{example}

The $FavoriteSites$ fluent is typical of many domain
properties/relations, such as the passengers in a plane, the students
in a class, or the cars parked in a parking lot, where we can think of
the relation as having a finite capacity, and where we can reassign
the objects that are in it.
In some cases, the capacity bound may be difficult to pin down, e.g.,
the guests at a wedding, altough the capacity is by no means
unbounded.
As well, there are definitely examples where we need an unbounded
theory, e.g., to model a pushdown automata that can recognize a
particular context-free language.  The situation calculus is a very
expressive language that accomodates this, for instance, it has been
used to model Turing machines \cite{Reiter01-Book}.
One might arguably want an unbounded ``favorite sites'' menu or
contacts directory, although this seems hardly practical.
Another interesting question is how such capacity constraints might apply
to a complex agent such as a robot that is modeling its environment.
Clearly, such a robot would have limitations with respect to how many environment
features/objects/properties it can memorize and track.
Finally, 
note that the condition $\card{\{\vec{x} \mid
  \Phi_F^+(\vec{x},a,s)\}} \le \card{\{\vec{x} \mid
  \Phi_F^-(\vec{x},a,s)\}}$ is not a FO formula and it is
difficult (in fact, undecidable) in general to determine whether a basic action theory is
effect bounded.
But as our examples illustrate, there are many 
instances where it is easy to show that the bounded effects condition
holds.

\subsection{Fading Fluents Action Theories}

Fading fluents action theories  are based on the idea that
information over time loses strength and fades away unless it is
reinforced explicitly.
A fading fluents action theory with fading length given by a  natural number $\ell$ is an action theory where
a fluent $F(\vec{x},s)$ is defined  by making use of some auxiliary fluents $F_{i}(\vec{x},s)$, for $0\le i \le \ell$ where
$F(\vec{x},s) \doteq \bigvee_{0\le i \le \ell} F_{i}(\vec{x},s)$
and the auxiliary fluents have successor state axioms of the following special form:
\[F_{\ell}(\vec{x},do(a,s)) \equiv \Phi_F^+(\vec{x},a,s) \land \card{\{\vec{x} \mid \exists a.\Phi_F^+(\vec{x},a,s)\}} < b\] 
and for $0\le i < \ell$ we have:
\[F_{i}(\vec{x},do(a,s)) \equiv \lnot \Phi_F^+(\vec{x},a,s) \land
F_{i+1}(\vec{x},s)\land \lnot \Phi_F^-(\vec{x},a,s).\]
Thus, tuples are initially added to $F_{\ell}$, and progressively lose their strength,
moving from $F_i$ to $F_{i-1}$ each time an action occurs that does not
delete or re-add them; eventually they move out of $F_0$ and are forgotten.
Note that:
\begin{itemize}
\item Technically, a fading fluents action theory is a basic action
  theory having as fluents only the auxiliary fluents.

\item It is simple to obtain a fading fluent version of any basic
  action theory.
\item It is often convenient to include explicit
  $\refresh_F(\vec{x})$ actions, whose effect, when applied to a situation $s$, is simply to make $F_\ell(\vec{x}, do(\refresh_F(\vec{x},s)))$ true, and $F_i(\vec{x}, do(\refresh_F(\vec{x},s)))$ false for $0\le i< \ell$.
Similarly it may be convenient to include $\forget_F(\vec{x})$ actions, whose effect is to make $F_i(\vec{x}, do(\forget_F(\vec{x},s)))$ false, for all $i$.
\end{itemize}

\begin{theorem}
  Let $\D$ be a fading fluents action theory with fading length $\ell$
  and initial database $\D_0$ such that $\D_0\models Bounded_b(S_0)$,
  for some $b$.  Then, $\D$ is bounded by $b$.
\end{theorem}
\begin{proof}
By induction on executable situations. 
For the base case,  we have that initially for each fluent, we have at
most $b$ facts, hence $S_0$  is bounded by $b$.
For the inductive case, by the inductive hypothesis we have   that $Bounded_{b}(s)$. Now, take an arbitrary action $a(\vec{t})$, and an arbitrary fluent $F$. Then: \myi $Bounded_{F_\ell,b}(do(a(\vec{t}),s))$, since positive effects are bounded by $b$ in its successor state axiom; and \myii for all $0\le i< \ell$, since $F_{i}$ depends on $F_{i+1}$ in the previous situation in its successor state axioms,  we have that $Bounded_{F_{i},b}(do(a(\vec{t}),s))$ since $Bounded_{F_{i+1},b}(s)$ and in the worst case the whole extension of $F_{i+1}$ in $s$ is carried over to $F_{i}$ in $do(a(\vec{t}),s)$. 
\end{proof}

\begin{example}
Imagine a sort of ``vacuum cleaner world'' where a robotic vacuum
cleaner may clean a room or region $r$ \cite{RussellNorvig}.  If a room/region is used,
then it becomes unclean.  We could model this using a fluent
$IsClean(r,s)$ with the following successor state axiom:
\[\begin{array}{l}
IsClean(r,do(a,s)) \equiv
a = clean(r) \lor IsClean(r,s) \land \neg a = use(r)
\end{array}\]
Clearly, cleanliness is a property that fades over time.
By applying the proposed transformation to this specification, we
obtain the following:
\[IsClean_{\ell}(r,do(a,s)) \equiv a = clean(r) \land 1 < b\]
and for $0\le i < \ell$ we have:
\[\begin{array}{l}
IsClean_{i}(r,do(a,s)) \equiv a \neq clean(r) \land
IsClean_{i+1}(r,s)\land a \neq use(r)
\end{array}\]
This is a somewhat more realistic model where after $\ell$ steps, we
forget about a room being clean.
\end{example}

\begin{example}
Consider a robot that can move objects around.  We might model this
using a fluent $At(objet,location,s)$ with the following successor state axiom:
\[\begin{array}{l}
At(o,l,do(a,s)) \equiv
a = moveTo(o,l) \lor a = observe(o,l) \lor {} \\
\hspace{2em} At(o,l,s) \land {} a \neq takeAway(o)\land {}\\
\hspace{2em}\neg \exists l'.l' \neq l \land (a = moveTo(o,l') \lor  a = observe(o,l'))\\
\end{array}\]
Here, $moveTo(o,l)$ represents the robot's moving object $o$ to
location $l$.  We also have an action $observe(o,l)$ of observing that
object $o$ is at location $l$, a kind of exogenous action that might be
produced by the robot's sensors.
As well, we have another exogenous action $takeAway(o)$,
representing another agent's taking object $o$ to an unknown
location $l$.  If the world is dynamic, most objects would not remain
where they are indefinitely, even if the robot is unaware of anyone
moving them.  
By applying the proposed transformation to this specification, we
obtain a theory where information about the location of objects fades
unless it is refreshed by the robot's observations or actions.
After $\ell$ steps, the robot forgets the location of an object 
it has not observed or moved; moreover, this happens immediately
if the object is taken away by another agent.
\end{example}

\begin{example}
As a final example, consider a softbot that keeps track of which hosts
are online.  We might model this using a fluent $NonFaulty(host,s)$
with the following successor state axiom:
\[\begin{array}{l}
NonFaulty(h,do(a,s)) \equiv
a = pingS(h) \lor NonFaulty(h,s) \land a \neq pingF(r)
\end{array}\]
Here the action $pingS(h)$ means that the host $h$ has been pinged
successfully, and the action $pingF(h)$ means that the host $h$ has
not responded to a pinging within the allocated time.  As time passes,
we may not want to assume that currently non-faulty hosts remain 
non-faulty.  If we apply the proposed transformation to this specification, we
obtain a theory where information about hosts being non-faulty fades.
The agent must periodically ping the host successfully to maintain its knowledge
that the host is non-faulty.
\end{example}

An interesting natural example of such fading representations is the
pheromones left by insects.  Note that it is also possible to model
fading with time as opposed to fading with the number of actions,
though in this case we have to bound how many actions can occur
between clock ticks.


\section{Expressing Dynamic Properties}
\label{sec:mucalc}

To express properties about Situation Calculus action theories, we
introduce a specific logic, inspired by the
$\mu$-calculus~\cite{Emerson96,BS07}, one of the most
powerful temporal logics, subsuming both linear time logics, such as
Linear Temporal Logic (LTL) \cite{PnuelliLTL97} and
Property-Specification Language (PSL) \cite{EiFi06}, and branching time logics such
as Computational Tree Logic CTL \cite{ClarkeE81CTL} and
CTL$^*$~\cite{EmHal83CTL*}. 
The main characteristic of the $\mu$-calculus is its ability to express directly
least and greatest fixpoints of (predicate-transformer) operators formed using
formulae relating the current state to the next one. By using such fixpoint
constructs one can easily express sophisticated properties defined by induction
or co-induction. This is the reason why virtually all logics used in
verification can be considered as fragments of $\mu$-calculus.
Technically, the $\mu$-calculus separates local properties, asserted on the current
state or on states that are immediate successors of the current one, from
properties talking about states that are arbitrarily far away from the current
one \cite{BS07}. The latter are expressed through the use of fixpoints.
Our variant of the $\mu$-calculus is able to express first-order properties over situation.
At the same time, it allows for a controlled form of first-order quantification
across situations, inspired by \cite{DBLP:conf/pods/HaririCGDM13}, 
where the quantification ranges over objects that persist in the
extension of some fluents across situations. 


Formally, we define the logic \mymu as:
\[\begin{array}{l}
  \Phi ::=  \varphi  \mid \lnot \Phi \mid \Phi_1 \land \Phi_2 \mid \exists
    x. \inadom(x) \land \Phi \mid {}\\
  \hspace{2.5em} \inadom(\vec{x}) \land \DIAM{\Phi}  \mid
   \inadom(\vec{x}) \land\BOX{ \Phi}\mid  Z \mid \mu Z.\Phi
\end{array}
\]
In addition, we use the usual FOL abbreviations for $\lor$, $\limp$,
$\equiv$, and $\forall$, plus the standard $\mu$-calculus abbreviation 
$\nu Z. \Phi = \neg \mu Z. \neg \Phi[Z/\neg Z]$.
Let us comment on some aspects of \mymu:

\begin{itemize}

\item $\varphi$ in the expression above is an arbitrary (possibly open) uniform
  \emph{situation-suppressed} (i.e., with all situation arguments in
  fluents suppressed) situation calculus FO formula, in which the only
  constants that may appear are those explicitly mentioned in the
  situation calculus theory beyond $\D_{uno}$, i.e., those occurring in 
  $\D_{poss} \cup \D_{ssa} \cup \D_0$.\footnote{
Clearly, we can get around this assumption by adding  to the
    initial situation description, a new ``dummy'' fluent that holds for a bounded number of constants.
}
Observe that quantification inside $\varphi$ is not subject to any restriction;
in particular, $\inadom(\cdot)$ is not required. 

\item The boolean connectives have their usual meaning. Quantification
  over individuals in $\exists
    x. \inadom(x) \land \Phi $ and $ \forall
    x. \inadom(x) \limp \Phi$ (i.e., $\lnot\exists
    x. \inadom(x) \land \lnot\Phi $) has the expected meaning, with the proviso that
  individuals over which quantification ranges must belong to the
active domain of the current situation, i.e., belong to the extension
of some fluent in the current situation, as required by $\inadom(\cdot)$.

\item Intuitively, the use of $\inadom(\cdot)$ in \mymu ensures that objects
are only considered in quantification across situations if they persist along the
system evolution, while the evaluation of a formula with objects that
are not present in the current extension of the fluents trivially
evaluates to either false for $\exists$ or true for $\forall$.
In particular:
\begin{itemize}

\item $\inadom(\vec{x})\land \DIAM{\Phi}$ denotes the set of
  situations $s$ such that for \emph{some} action $a$ that is
  executable in $s$, we have that $\Phi$ holds in $do(a,s)$, with the
  variables occurring free in $\Phi$, $\vec{x}$, assigned to objects
  that are in the active domain of the current situation $s$.

\item $\inadom(\vec{x})\land \BOX{\Phi}$ denotes those situations $s$
  such that for \emph{all} actions $a$ that are executable in $s$, we
  have that $\Phi$ holds in $do(a,s)$ with the variables occurring
  free in $\Phi$ are assigned to objects that are in the active domain
  of the current situation $s$.

\item $\inadom(\vec{x})\limp \DIAM{\Phi}$ (i.e., $\neg (\inadom(\vec{x}) \land \BOX{\neg \Phi})$) denotes those
situations $s$ such that for \emph{some} action $a$ that is executable in $s$, we
have that $\Phi$ holds in $do(a,s)$ \emph{as long as} the variables occurring free in $\Phi$ are assigned to objects that are in the active domain of the current situation $s$.

\item$\inadom(\vec{x})\limp \BOX{\Phi}$ (i.e., $\neg (\inadom(\vec{x}) \land \DIAM{\neg \Phi})$) denotes those
situations $s$ such that for \emph{all} actions $a$ that are executable in $s$, we
have that $\Phi$ holds in $do(a,s)$ \emph{as long as} the variables occurring free in $\Phi$ are assigned to objects that are in the active domain of the current situation $s$.
\end{itemize}

\item  $Z$ is an SO (0-ary) predicate variable.

\item  $\mu Z.\Phi$ and $\nu Z.\Phi$ are  \emph{fixpoint formulas} and denote
respectively the \emph{least}  and the \emph{greatest fixpoint} of the formula
$\Phi$ seen as a predicate transformer $\lambda Z.\Phi$. To guarantee the existence of such fixpoints, as usual in the $\mu$-calculus, formulae of the form $\mu
Z.\Phi$ and $\nu Z.\Phi$ must satisfy \emph{syntactic monotonicity}
of $\Phi$ with respect to $Z$, which states that every occurrence of the variable
$Z$ in $\Phi$ must be within the scope of an even number of negation
symbols.  
\item $\mu Z.\Phi$ and  $\nu Z.\Phi$ may contain free individual
  variables, which are those 
of $\Phi$; technically these act as \emph{parameters} of the fixpoint formula, i.e., the value of fixpoints $\mu Z.\Phi$ and $\nu Z.\Phi$  is determined only once an assignment to the free individual variables is given, see, e.g., \cite{Libkin04} (chap.~10).
\item Finally, with a slight abuse of notation, we write
$\inadom(x_1,\ldots,x_n) = \bigwedge_{i \in \{1,\ldots,n\}} \inadom(x_i)$, and we assume that
in
$\inadom(\vec{x}) \land \DIAM{ \Phi}$ and $\inadom(\vec{x}) \land \BOX{\Phi}$,
the variables $\vec{x}$ are exactly
the free individual variables of $\Phi$, after we have substituted each
bound predicate variable $Z$ in $\Phi$ by the
corresponding binding fixpoint formula $\mu Z.\Phi'$ or $\nu Z.\Phi'$.
\end{itemize}

We can express arbitrary temporal/dynamic properties using least and
greatest fixpoint constructions.  For instance, to say that it is
possible to eventually achieve $\varphi$, where $\varphi$
is a closed situation suppressed formula, we
use the least fixpoint formula
$
\mu Z. \varphi \lor \DIAM{Z}
$.
Similarly, we can use a greatest fixpoint formula $\nu Z. \varphi
\land \BOX{Z}$  to express that $\varphi$ must always hold.

\begin{example}
  We can give several examples of properties that we may want to
  verify for the warehouse robot domain of Example 1.  First, suppose
  that we want to say that it is possible to eventually
  have shipped all items that are in the factory.  This
  can be expressed in our language as a least fixpoint formula:
\[
\mu Z. \neg\exists x \exists l. At(x,l) \lor \DIAM{Z}
\]
This formula, let's call it $\Phi_{eg9}$, corresponds to the CTL
formula $EF \neg\exists x \exists l. At(x,l).$ In the above, we rely
on the fact that if there are no items left in the factory, then all
items that were there must have been shipped.  It is easy to check
that the theory of Example 1, $\D_1$, entails that this formula holds
in the initial situation $S_0$, formally $\D_1 \models \Phi_{eg9}$.
In fact, we can also show that the above property always holds:
\[
\D_1 \models \nu Z. \Phi_{eg9} \land \BOX{Z}.
\]
This corresponds to the CTL formula $AG EF \neg\exists x \exists l. At(x,l).$
Note that more generally, a formula $\mu Z. \varphi \lor \DIAM{Z}$, i.e.,
$EF \varphi$ in CTL, represents an instance of a
planning problem; it is entailed by a theory if there exists an
executable sequence of actions such that the goal  $\varphi$
holds afterwards.

A second example property that we may want to verify is that it is
possible to eventually have all items shipped out of the factory and
then later to eventually have all locations filled with items.  This
can be expressed as follows:
\[
\D_1 \models \mu Z .[(\neg \exists l \exists x. At(x,l)) 
\land \mu Z .(\forall l. IsLoc(l) \limp  \exists x. At(x,l)) \lor \DIAM{Z})] \lor \DIAM{Z}
\]
or equivalently in CTL notation
\[
\D_1 \models EF( (\neg \exists l \exists x. At(x,l)) \land EF (\forall
l. IsLoc(l) \limp \exists x . At(x,l))).
\]

Our next example concerns a safety property; we can show that it is
always the case that if an item is at the shipping dock it can be
moved away or shipped out next:
\[
\D_1 \models \nu Z.  [(\exists x. At(x,ShipDock))  \supset \DIAM{(\neg
  \exists x. At(x,ShipDock)) }]\land \BOX{Z}
\]
or equivalently in CTL notation
\[
\D_1 \models AG [(\exists x. At(x,ShipDock))  \supset \DIAM{(\neg
  \exists x. At(x,ShipDock)) }].
\]
However, this is not the case for other locations, as it is possible
for all locations to become occupied, at which point the agent must
ship the item at the shipping dock before it can transfer the item at
the location of interest there:
\[
\D_1 \models \neg \nu Z.  [\forall l. (\inadom(l) \limp 
(\exists x. At(x,l) \limp 
(\inadom(l) 
\limp \DIAM{(\neg \exists x. At(x,l))})))] \land \BOX{Z}
\]
which simplifies to (also observing that $\exists x. At(x,l)$ implies  $\inadom(l)$):
\[
\D_1 \models \neg \nu Z.  [\forall l. 
(\exists x. At(x,l)
\limp \DIAM{(\neg \exists x. At(x,l)) })]\land \BOX{Z}.
\]

But it is always possible to clear a location in two steps:
\[
\D_1 \models \nu Z.  [\forall l.  
   (\exists x. At(x,l)  \limp
        (\DIAM{
           (\inadom(l) \land \DIAM{(\neg
  \exists x. At(x,l)) })}))]\land \BOX{Z}
\]
The above involves quantification across situations, and we require
the location involved to persist (it trivially does).

Now, let's consider another example were we quantify across situations.
We may want to say that it is always the case that if an item is in
the warehouse, it is possible to have it persist until it is
eventually shipped out:
\[
\D_1 \models \nu Z.  [\forall x. (\exists l. At(x,l))  \supset 
\mu Z .(\neg \exists l. At(x,l)) \lor \inadom(x) \land \DIAM{Z}]
\land \BOX{Z}.
\]
or equivalently in CTL notation
\[
\D_1 \models AG [\forall x.(\exists l. At(x,l))  \supset
EF \neg \exists l. At(x,l)].
\]
Note that the weaker property that it is always the case that if an item is in
the warehouse, it is possible to have it shipped out eventually \emph{if it
persists} also holds:
\[
\D_1 \models \nu Z.  [\forall x. (\exists l. At(x,l))  \supset 
\mu Z .(\neg \exists l. At(x,l)) \lor (\inadom(x) \supset \DIAM{Z})]
\land \BOX{Z}.
\]

Finally, consider the property that if an item is eventually shipped,
it is possible for it to eventually come back:
\[
\forall x.\exists l. At(x,l))  \supset AG [\neg \exists l. At(x,l) \supset
EF \exists l. At(x,l)].
\]
We cannot express this property in \mymu\ because $x$ does not persist after it has been shipped.  The closest translation
\[\begin{array}{l} 
\forall x. \inadom(x) \land \exists l. At(x,l) \limp \\
\qquad\nu Z.  [\neg
\exists l. At(x,l) \supset \mu Z .(\exists l. At(x,l)) \lor \inadom(x)
\land \DIAM{Z}] \land \BOX{Z}.
\end{array}\]
is always false because if $x$ is not at some location, then  it is not in
the active domain and $\inadom(x)$ is false.
%
\end{example}

Next we turn to semantics.  Since $\mymu$ contains formulae with free individual and 
predicate  variables, given a model $\M$ of an action theory $\D$
with object domain $\Delta$ and situation domain $\S$,  we introduce a \emph{valuation} $(\vfo,\vso)$ formed by an
\emph{individual variable valuation} $\vfo$ which maps each individual variable $x$ to an object $\vfo(x)$ in $\Delta$, and a \emph{parametrized predicate variable valuation}
$\vso$, which, given the valuation of the individual variables $\vfo$, maps each predicate variable to $Z$ to subset $\vso(\vfo,Z)$ of situations
in $\S$ (notice that for each individual variable valuation $\vfo$ the
mapping may change).
Given a valuation $(\vfo,\vso)$, we denote by $(\vfo,\vso)[x/d]$ the
valuation $(\vfo',\vso')$ such that:
\begin{inparaenum}[\itshape(i)]
\item for every individual variable $y\neq x$ we have  $\vfo'(y)=\vfo(y)$ and $\vfo'(x)=d$, 
\item for every predicate variable $Z$ we have $\vso'(\vfo',Z)= \vso(\vfo',Z)$
\end{inparaenum}
Sometimes we also use the notation  $\vfo[\vec x/\vec d]$ to denote $\vfo'$
such that for every individual variable $y\neq x$ we have  $\vfo'(y)=\vfo(y)$ and $\vfo'(x)=d$.
To express that $\vfo$ assigns the
values $\vec d$ to the variables $\vec x$, we use the notation
$\vec x/\vec d$.
Analogously, we denote by $(\vfo,\vso)[Z/\E]$ the
valuation $(\vfo',\vso')$ such that:
\begin{inparaenum}[\itshape(i)]
\item for every individual variable $x$ we have  $\vfo'(x)=\vfo(x)$, 
\item for every predicate variable $Y\neq Z$ we have
  $\vso'(\vfo',Y)= \vso(\vfo,Y)$, and for $Z$ we have
  $\vso'(\vfo',Y)= \E$.
\end{inparaenum} 
Also we denote by $\adom^{\M}(s)$, the {\em active (object) domain} of
situation $s$ in the model $\M$, which is the set of all objects
occurring in some $F^\M(s)$ ($F\in\F$) or as the denotation in $\M$ of
a constant in the set $C$ of object constants occurring in
$\D_{poss}\cup \D_{ssa}\cup \D_0$.
Then we assign semantics to formulae by associating to a model $\M$, and a valuation $(\vfo,\vso)$
an \emph{extension function} $\MODA{\cdot}$, which maps $\mymu$
formulae to subsets of $\S$ as inductively defined as follows (for clarity, we interpret
explicitly also the abbreviation $\nu Z.\Phi$):
\[\begin{array}{lcl}
\MODA{\varphi} &= & \{s \in \S\mid \M,\vfo\models \varphi[s] \}\\
\MODA{\lnot \Phi} &= & \S - \MODA{\Phi} \\
\MODA{\Phi_1 \land \Phi_2} &= &\MODA{\Phi_1}\cap\MODA{\Phi_2}\\
\MODA{\exists x.\,\inadom(x)\land\Phi} &=&
      \{s\in\S \mid\exists d\in \adom^{\M}(s).\,
      s \in \MODAX{\Phi}{[x/d]}\}\\
 \MODA{\inadom(\vec{x})\land \DIAM{\Phi}} &=&
      \{s\in\S \mid \vec{x}/\vec{d} \in \vfo \mbox{ and } \vec{d}\subseteq\adom^{\M}(s) \mbox{ and }\\
    & & \quad\exists a. \, (a,s)\in Poss^\M \mbox{ and } do^\M(a,s) \in \MODA{\Phi}\} \\

    \MODA{\inadom(\vec{x})\land \BOX{\Phi}} &=&
      \{s\in\S \mid \vec{x}/\vec{d} \in \vfo \mbox{ and } \vec{d}\subseteq\adom^{\M}(s) \mbox{ and }\\
    & & \quad\forall a.\, (a,s)\in Poss^\M \mbox{ implies } do^\M(a,s) \in \MODA{\Phi}\} \\
\MODA{Z} &= & \vso(\vfo,Z)\\
\MODA{\mu Z.\Phi} &= & \bigcap\{\E\subseteq\S \mid \MODAZ{\Phi}{[Z/\E]}
\subseteq \E \}\\
\MODA{\nu Z.\Phi} &= & \bigcup\{\E\subseteq\S \mid \E \subseteq \MODAZ{\Phi}{[Z/\E]}\}
\end{array}
\]
Notice that given a (possibly open) uniform situation-suppressed
situation calculus formula $\varphi$, slightly abusing notation, we
denote by $\varphi[s]$ the corresponding formula with situation
calculus argument reintroduced and assigned to situation $s$.

Intuitively, the extension function $\MODA{\cdot}$ assigns the
following meaning to the $\mymu$ constructs:%
\footnote{ By mentioning situations explicitly, it is also possible to
  define these operators directly in second-order logic
as follows \cite{DeTR97}:
\[
\begin{array}{rcl}
\mu Z.\Phi[s] &\equiv& \forall Z.(\forall \hat{s}.\Phi[\hat{s}] \supset Z(\hat{s})) \supset  Z(s)\\
\nu Z.\Phi[s] &\equiv& \exists Z.(\forall \hat{s}.Z(\hat{s}) \supset \Phi[\hat{s}]) \land  Z(s)
\end{array}
\]
Note that $\Phi$ may contain free individual and predicate
variables, and indeed these remain free in $\mu Z.\Phi$ and
$\nu Z.\Phi$.  In this paper, we prefer to leave the situation
implicit to allow for interpreting formulas over arbitrary transition
systems, including finite ones, and hence relating our logic to
standard $\mu$-calculus.}

\begin{itemize}
\item The extension of $\mu Z.\Phi$ is the \emph{smallest subset}
$\E_\mu$ of situations such that, assigning to $Z$ the extension
$\E_\mu$, the resulting extension of $\Phi$ is contained in
$\E_\mu$ (with the assignments of the individual variables and the other predicate variables given by $\vfo$ and $\vso$, respectively). That is, the extension of $\mu Z.\Phi$ is the \emph{least
fixpoint} of the operator $\MODAZ{\Phi}{[Z/\E]}$. Notice that for each valuation of the free individual variables in $\Phi$ this operator will be different: the free variables act as \emph{parameters} of the predicate transformer $\lambda Z.\Phi$.

\item Similarly, the extension of $\nu Z.\Phi$ is the \emph{greatest
subset} $\E_\nu$ of situations such that, assigning to $Z$ the extension
$\E_\nu$, the resulting extension of $\Phi$ contains $\E_\nu$. That
is, the extension of $\nu Z.\Phi$ is the \emph{greatest fixpoint} of
the operator $\MODAZ{\Phi}{[X/\E]}$. 
\end{itemize}

Notice also that when a $\mymu$ formula $\Phi$ is closed, its
extension $\MODA{\Phi}$ does not depend on the 
valuation $(\vfo,\vso) $.  In
fact, the only formulas of interest in verification are those that are
closed.

\paragraph{Observation 1}
Observe that we do not have actions as parameters of $\BOX{\cdot}$ and
$\DIAM{\cdot}$. However we can easily remember the last action
performed, and in fact a finite sequence of previous actions. To do
this, for each action type $A(\vec{x})$, we introduce a fluent $Last_A(\vec{x},s)$ with successor state axiom:
\[Last_A(\vec{x},do(a,s)) \equiv a=A(\vec{x})\]
We can also remember the second last action by introducing  fluents $SecondLast_A(\vec{x},s)$ with successor state axioms:

\[SecondLast_A(\vec{x},do(a,s)) \equiv Last_A(\vec{x},s)\]
Similarly for the third last action, etc.

In this way we can store a finite suffix of the history in
the current situation and write FO formulas relating the individuals
in the parameters of actions occurring in the suffix. For example, we can
write (assuming  for simplicity that the mentioned fluents have all the same arity):
\[\mu Z. (\exists \vec{x} . Last_A(\vec{x})\land SecondLast_B(\vec{x})) \lor \DIAM{Z},\]
i.e., it is possible to eventually do $B(\vec{x})$ followed by
$A(\vec{x})$ for some $\vec{x}$.

\paragraph{Observation 2} 
Observe that while our $\mymu$ allows for 
quantification over objects that persist across situations, 
the  expressiveness of bounded action theories means that we can often to avoid its use. 
For instance, we can
easily introduce a finite number of ``registers'', i.e., fluents that
store only one tuple, which can be used to store and refer to tuples
across situations.  We can do this by introducing fluents
$Reg_i(\vec{x},s)$ and two actions $setReg_i(\vec{x})$ and $clearReg_i$ to
set and clear the register $Reg_i$ respectively. These are axiomatized
as follows:
\[\begin{array}{l}
Reg_i(\vec{x}, do(a,s)) \equiv a=setReg_i(\vec{x}) \lor {}\\
\qquad  Reg_i(\vec{x},s) \land a \neq clearReg_i\\[.5ex]
Poss(setReg_i(\vec{x}),s) \equiv \lnot\exists\vec{x}. Reg_i(\vec{x},s)\\
Poss(clearReg_i,s) \equiv \exists\vec{x}. Reg_i(\vec{x},s)
\end{array}
\]
For example, we can write (assuming for simplicity that the mentioned fluents have all the same arity):
\[\mu Z. (\exists \vec{x} . Reg_i(\vec{x})\land F(\vec{x})\land \DIAM{\exists \vec{y} . Reg_i(\vec{y})\land F'(\vec{y})}) \lor \DIAM{Z}\]
This formula says that there exists a sequence of actions where eventually the tuple referred to by register $i$ has property $F$ and there is an action after which it has property $F'$. 
Note also that this approach can be used to handle some cases of
quantification over objects that don't persist across situations.

\section{Verification  of Bounded Action Theories with Complete Information on $S_0$}
\label{sec:verification}
We now show that verifying $\mymu$ properties against bounded action
theories is decidable.
In this section we focus on action theories with complete information
on the initial situation. The case of incomplete information is addressed in the next section.
In particular, we assume that the extension of all fluents in the
initial situation $S_0$ is given as a (bounded) database.  We further assume that the domain of
interpretation for objects $\Delta$ is also given. Notice that, as a consequence of the
presence of infinitely many object constants and the unique name
assumption on them $\D_{uno}$, such an object domain $\Delta$ must be infinite.\footnote{By the way in case of action theories
  with a given finite object domain, verification becomes easily
  reducible to model checking, since the corresponding situation
  calculus model it is bisimilar to a finite propositional transition
  system.} 
As a result of these two assumptions, we have that the action theory $\D$ admits only one model $\M_\Delta$
\cite{PirriR:JACM99}, which, with a little abuse of terminology, we call \emph{the} model of the action theory $\D$ (though in order to define it we need $\Delta$ as well). 

Our main result is the following.
\begin{theorem}\label{th:dec-verif}
 Let $\D$ be a bounded action theory with initial situation described
  by a (bounded) database and with infinite object domain $\Delta$, and let $\Phi$ be a closed $\mymu$ formula. Then
  checking whether $\D\models \Phi$ is decidable.
\end{theorem}

The proof is structured as follows. Firstly, we show that 
actions terms can be eliminated from $\mymu$ formulas with out loss of generality (cf.~Section~\ref{sec:action-suppression}).
Exploiting this, we show that  
only the fluent extensions in each situation and not situations themselves are relevant when evaluating $\mymu$ formulas (cf.~Section~\ref{sec:situation-suppression}).
In this step, we also prove that
checking FO formulas and answering FO queries 
\emph{locally}, i.e., on a given situation, are, 
respectively, decidable and effectively computable, under boundedness.

Then, based on the observations above,  we introduce \emph{transition systems} 
as alternative structures  (to the models of 
situation calculus action theories), over which $\mymu$ formulas can be evaluated.
Transition systems are less rich than the models of situation calculus action
theories, as they do not reflect, in general, the structure of the situation tree.
Yet, they can accommodate the information of models needed to evaluate  
$\mymu$ formulas (cf.~Section~\ref{sec:mu-calc-ts} and~\ref{sec:ind-ts}). 
In this step, we define the notion of
\emph{persistence-preserving bisimulation}, i.e., a variant of
standard bisimulation which requires a certain kind of isomorphism to
exist between bisimilar states and their successors
(cf.~page~\pageref{def:bisim}), and prove that persistence-preserving
bisimilar transition systems preserve the truth-value of $\mu\L$
formulas (cf.~Theorem~\ref{thm:mymu-bisimulation}). This is a key step
in the proof, which allows us to reduce the verification of $\mymu$
formulas over an infinite transition system to that over a bisimilar
transition system that is finite.


In the third and fundamental step (Section~\ref{sec:finite-abstraction}),
we carry out a \emph{faithful abstraction} operation, and 
show how to actually construct 
a finite transition system that is persistence-preserving bisimilar to
the one, infinite, induced by \emph{the} model of the action theory
(cf.~Procedure~\ref{alg:tf} and Theorems~\ref{thm:termination} and
\ref{thm:bisimulation-finite-induced}). Finally, we prove that
verification is decidable on finite transition systems, thus on the
one induced by the model of the action theory
(cf.~Theorem~\ref{th:mymu-dec}).

The rest of this section details these steps.

\subsection{Suppressing Action Terms}
\label{sec:action-suppression}
Under uniqueness of action names, domain closure for actions, 
and the fact that action types are finitely many, w.l.o.g.,
we can remove action terms from uniform situation calculus formulas.



\begin{theorem}\label{prop:sitcalc-fo-act}	
  For every, possibly open, situation calculus FO formula $\varphi(\vec{x}, s)$ uniform
  in $s$ and with free variables $\vec{x}$, all of object sort, 
  there exists a situation calculus formula $\varphi'(\vec{x}, s)$ 
  uniform in $s$, where no action terms occur, such that
$$\D_{ca}\models \forall(\varphi(\vec{x}, s) \equiv \varphi'(\vec{x}, s)).$$
\end{theorem}
\begin{proof}
By  induction on the structure of $\varphi$.
For $\varphi=F(\vec t,s)$, we have that by definition $\vec{t}$ can
only contain object terms so $\varphi'=\varphi$, else
$\varphi'=\varphi$.
For $\varphi=A(\vec y)=A'(\vec{y'})$, with
$\vec x\subseteq \vec y\cup\vec y'$, if $A=A'$, then
$\varphi'=\vec y=\vec{y'}$, else $\varphi'=\bot$.
The case of boolean connectives is straightforward.
%
If $\varphi=\exists a. \phi(\vec x,a,s)$, 
consider the formula 
$\varphi''=\bigvee_{A\in \A}\exists \vec y_A.\phi_A(\vec x, \vec y_A,s)$,
with $\phi_A$ obtained from $\phi(\vec x,a,s)$, by replacing each occurrence
of $a$ with $A(\vec y_A)$, where $\vec y_A$ are fresh variables.
We obviously have: $\D_{ca}\models \forall(\varphi\equiv \varphi'')$.
Now, for each $\phi_A$, let $\phi'_A$ be a formula
containing  no action terms, such that~$\D_{ca}\models \forall(\phi_A\equiv \phi_A')$.
By induction hypothesis, such a $\phi_A'$ exists.
Finally, let $\varphi'=\bigvee_{A\in\A}\exists \vec y_A.\phi_A'(\vec x,\vec y_A,s)$.
Clearly, $\varphi'$ contains no action terms and 
is uniform in $s$.
By considering unique name axioms for actions
and domain closure for action types ($\D_{ca}$),
we can see that $\D_{ca}\models \forall(\varphi''\equiv\varphi')$.
Thus, since $\D_{ca}\models \forall(\varphi\equiv \varphi'')$,
the thesis follows, i.e., $\D_{ca}\models \forall(\varphi\equiv \varphi')$.
\end{proof}


Such a result immediately extends to $\mymu$, since in  $\mymu$ formulas only uniform (situation suppressed) situation calculus FO subformulas can occur. 
\begin{theorem}\label{th:mymu-drop-act}
Any $\mymu$ formula $\Phi$ can be rewritten into an 
equivalent $\mymu$ formula $\Phi'$, where no action terms occur, such that 
$\D_{ca}\models \forall(\Phi\equiv \Phi')$.
\end{theorem}
On the basis of this theorem, w.l.o.g.,  we will always rewrite  $\mymu$ formulas so as that  
 actions do not occur in them.

\subsection{Suppressing Situation Terms}
\label{sec:situation-suppression}

Since the FO components of $\mymu$ formulas are situation-suppressed,
situations are obviously irrelevant when checking $\mymu$ formulas;
more precisely, the FO components (thus the whole logic) are sensitive
only to the interpretation of fluents (and constants) at each
situation, while the situations themselves are not relevant.  The
impact of this observation on the evaluation of $\mymu$ formulas in
the general case will become evident in Section~\ref{sec:ind-ts}.
Here, we focus on the \emph{local} evaluation of FO components (on the
interpretation of a single situation), or more specifically of FO
situation calculus formulas uniform in $s$, and present some notable
results that, besides being interesting \emph{per se}, will be useful
later on.

Given a basic action theory $\D$, we denote by $\F$ the set of its
fluent symbols and by $C$ the (finite) set of constants in $\N$
explicitly mentioned in $\D$, beyond $\D_{uno}$.  Then given a model
$\M$ of $\D$ with object domain $\Delta$ and a situation $s$, it is
natural to associate $s$ with a FO interpretation
$\I_\M(s)\doteq\tup{\Delta,\cdot^\I}$, where: \myi for every $c\in C$,
$c^\I = c^\M$ and \myii for every (situation-suppressed) fluent $F$ of
$\D$, $F^\I=\set{\vec d \mid \tup{\vec d,s}\in F^\M}$.
The following result is an obvious consequence of the definitions above.
\begin{theorem}\label{th:mod-int}
  For any possibly open FO situation-suppressed situation calculus
  formula $\varphi$ uniform in $s$, any situation $s$ and any object
  variable valuation $\vfo$, we have that $\M,v\models \varphi[s]$ if
  and only if $\I_\M(s),v\models \varphi$.
\end{theorem}
In other words, when evaluating a uniform FO situation-calculus 
formula on a situation, one needs only focus on the interpretation relative
to the situation of interest.

Next, we show that, for bounded action theories, we have decidability of 
evaluation of FO formulas in spite of the object domain being infinite. 
Even more, we obtain that we can compute the answers to FO queries on
specific situations.  Notice that the latter result is not obvious, 
in that the object domain is infinite and, thus, so could be the answer.
Importantly, these results imply that we can check action 
executability and compute the effects of action executions, two facts
that we will strongly leverage on when checking $\mymu$ formulas.

We begin by showing some results concerning the decidability of FO formula evaluation in an
interpretation with finite predicate extensions, but infinite
domain.  More precisely, we consider a finite set $\F$ of predicate
symbols (situation-suppressed fluents) and a finite set $C$ (a subset
of $\N$) of constant symbols, a (FO) \emph{interpretation} $\I$, over
an infinite domain $\Delta$ is a tuple $\I=\tup{\Delta,\cdot^\I}$,
where $\cdot^\I$ assigns an extension $F^\I$ over $\Delta$ to each
predicate symbol $F\in\F$, and a distinct object $c^\I\in\Delta$ to every
constant in $C$.
The \emph{active domain} of an interpretation $\I$, denoted
$\adom(\I)$ is the set of all the individuals occurring in the
extension of some fluent $F\in \F$, or interpreting some constant
$c\in C$, in $\I$.  Moreover, for simplicity, we assume that all
constants mentioned in FO formulas of interest belong to $\C$.

First, let us recall a classical result saying that FO formulas (with no function symbols 
other than constants) can always be rewritten as formulas with quantified
variables ranging only over the active domain of the interpretation.
For an interpretation $\I=\tup{\Delta,\cdot^{\I}}$, 
we define the \emph{restriction of $\I$ 
to its active domain} as the interpretation 
$\arestrict{\I}=\tup{\adom(\I),\cdot^{\I}}$.
In words, 
$\arestrict{\I}$ is the same interpretation as $\I$,
except that the object domain is replaced by 
the active domain.

\begin{theorem}[Theorem~5.6.3 of~\cite{Libkin07}]\label{thm:dom-ind}
  For every FO formula $\varphi$, one can effectively compute a
  formula $\varphi'$, with quantified variables ranging only over the
  active domain, such that for any interpretation
  $\I=\tup{\Delta,\cdot^\I}$ with infinite domain $\Delta$,
  and any valuation $\vfo$, we have that $\I,\vfo\models \varphi$ if and only if
  $\arestrict{\I},\vfo\models \varphi'$.
\end{theorem}
This result says that checking whether $\I,\vfo\models\varphi'$
requires knowing only the interpretation function $\cdot^\I$ of $\I$,
while the interpretation domain $\Delta$ can be disregarded. In other
words $\varphi'$ is a \emph{domain-independent} formula \cite{AbHV95}.
One way to obtain domain-independent formulas is to avoid the use of
negation and instead use logical difference with respect to the active
domain. The above theorem says that it is always possible to transform a FO
formula to be evaluated over an infinite domain to a
domain-independent one to be evaluated over the active domain only (and
actually its proof gives an effective procedure to do so).


An immediate consequence of Theorem~\ref{thm:dom-ind} is that
if $\adom(\I)$ is finite, then checking whether $\I,\vfo\models \varphi$ is decidable,  
no matter whether the interpretation domain of $\I$ is finite or infinite.
Indeed, in the former case, decidability is obvious, while in the latter, 
one can simply check $\arestrict{\I},\vfo\models \varphi'$, which  
requires only lookups on the \emph{finite} extensions of fluents
and, in presence of quantified variables, iterating over the \emph{finitely many}
elements of the active domain. Thus, we have the following result.

\begin{theorem}\label{th:fin-sat}
  Given a possibly open FO formula $\varphi$ and an interpretation
  $\I=\tup{\Delta,\cdot^\I}$ with infinite $\Delta$, if $\adom(\I)$ is
  finite, then, for any valuation $\vfo$, checking whether
  $\I,\vfo\models\varphi$ is decidable.
\end{theorem}
\begin{proof}
See discussion above.
\end{proof}

Theorem~\ref{th:fin-sat} can be lifted 
to computing all the valuations $\vfo$  such that $\I,\vfo\models\varphi$.
Let $\varphi$ be a FO formula
with free variables $\vec x$, and $\I=\tup{\Delta,\cdot^\I}$ a FO interpretation.
Then, the \emph{answer
on $\I$} to $\varphi$ is the relation 
$\varphi^{\I}\doteq \set{\vec d\in\vec\Delta\mid \I,\vfo\models\varphi\mbox{, for }\vfo(\vec x)=\vec d}$.
Sometimes, it is useful to fix the valuation of some variables $\vec x_{in}\subseteq \vec x$,
say $\vfo(\vec x_{in})=\vec d_{in}$, and then consider the answer to $\varphi$ under this partial assignment, that is, the relation
$\varphi^{\I}_{\vec x_{in}/\vec d_{in}}\doteq \set{\vec d_{out}\in \vec\Delta\mid 
\I,\vfo\models\varphi \mbox{, for } \vfo(\vec x_{in})=\vec d_{in}\mbox{ and } 
\vfo(\vec x\setminus \vec x_{in})=\vec d_{out}}$\footnote{$\vec x\setminus 
\vec x_{in}$ denotes the tuple obtained from $\vec x$ by projecting
out the components of $\vec x_{in}$.}.
The following theorem says that if $\I$ has an infinite domain
$\Delta$ but a finite active domain and the answer
$\varphi^{\I}_{\vec x_{in}/\vec d_{in}}$ is finite, then the objects
occurring in the answer come necessarily from either the active
domain, 
or the values assigned to $\vec x_{in}$
by $\vfo$.

\begin{theorem}\label{th:fin-ans}
Consider a FO formula $\varphi$ with free variables $\vec x$.
Let $\I$ be an interpretation with infinite $\Delta$ and finite active domain. 
If $\varphi^\I_{\vec x_{in}/\vec d_{in}}$ is finite, then 
$\varphi^\I_{\vec x_{in}/\vec d_{in}}\subseteq 
(\adom(\I)\cup \vec d_{in})^n$, 
where $n=\card{\vec x\setminus \vec x_{in}}$.
\end{theorem}
\begin{proof}
By contradiction. It can be easily proven that if $\I,\vfo\models\varphi$, for
$\vfo(x_i)=d_i\notin (\adom(\I)\cup \vec d_{in})$ and
$x_i\in \vec x\setminus \vec x_{in}$, 
then for any other valuation $\vfo'=\vfo[x_i/d'_i]$ such that 
$d'_i\in \Delta\setminus (\adom(\I)\cup \vec d_{in})$, we have that
$\I,\vfo'\models\varphi$. 
Since $\Delta$ is infinite and $\adom(\I)$ is finite, 
such $d_i'$ are infinitely many, thus 
$\varphi^\I_{\vec x_{in}/\vec d_{in}}$ is infinite. Contradiction.
\end{proof}
In other words, any ``new'' object, with respect to those in $\adom(\I)$,
occurring in the answer, must come  from $\vec d_{in}$.
A direct consequence of Theorems~\ref{th:fin-sat} and~\ref{th:fin-ans}
is that one can actually compute the answer on $\I$ to $\varphi$.
\begin{theorem}\label{th:fin-comp}
  Consider a FO formula $\varphi$ with free variables $\vec x$.  Let
  $\I=\tup{\Delta,\cdot^\I}$ be an interpretation with infinite
  $\Delta$ and finite active domain.  If, for some valuation $\vfo$
  such that $\vfo(\vec x_{in}) = \vec d_{in}$,
  $\varphi^\I_{\vec x_{in}/\vec d_{in}}$ is finite, then
  $\varphi^\I_{\vec x_{in}/\vec d_{in}}$ is effectively computable.
\end{theorem}
\begin{proof}
It suffices to record in $\varphi^\I_{\vec x_{in}/\vec d_{in}}$ all those tuples 
$\vec d_{out}$ 
such that for some $\vfo$ with $\vfo(\vec x_{in}) = \vec d_{in}$ 
and $\vfo(\vec x\setminus{\vec x_{in}}) = \vec d_{out}$, it is the case that
$\I,\vfo\models\varphi$.
Since by Theorem~\ref{th:fin-ans} such $\vec d_{out}$ are finitely many and can be obtained 
using values from $\adom(\I)\cup \vec d_{in}$, which is finite,
and, by Theorem~\ref{th:fin-sat}, checking whether $\I,\vfo\models\varphi$ is 
decidable, it follows that $\varphi^\I_{\vec x_{in}/\vec d_{in}}$ is computable. 
\end{proof}

These results find immediate application to the case of bounded action theories.
Indeed, bounded action theories guarantee that $\I_\M(s)$, in Theorem~\ref{th:mod-int},
is finite, (for $s$ executable). Thus, by Theorem~\ref{th:fin-sat}, 
for $\varphi$ and $\vfo$ as above, we have that
checking whether $\I_\M(s),\vfo\models\varphi$ is decidable.
A useful implication of this is that it is decidable to check whether
an action $A^\M(\vec o)$ is executable in a given situation $s$.
Indeed, this requires checking whether
$\M,\vfo\models\Poss(A(\vec x),s)$, with $\vfo(\vec x)=\vec o$, which,
by Theorem~\ref{th:mod-int}, is equivalent to
$\I_\M(s),\vfo\models\phi_A(\vec x)$, with $\phi_A(\vec x,s)$ the RHS
of the the precondition axiom of $A$, which, in turn, is decidable.
Moreover, Theorem~\ref{th:fin-comp}, can be used to show that for a
bounded action theory, the effects of executing an action at a given
situation, as determined the successor-state axioms, are computable
and depend only on $\I_\M(s)$ (and the action).
Indeed, we can exploit these results to get a sort of one-step regression
theorem in our setting ~\cite{PirriR:JACM99,Reiter01-Book}.
\begin{theorem}\label{th:ssa-fo}
Let $\M$ be a model of a bounded action theory $\D$, $s$ an 
executable situation, and $a=A^\M(\vec o)$ an action, with action type $A(\vec y)$.
Then, for any fluent $F$, there exists a situation-suppressed action-term-free
formula $\phi=\phi(\vec x,\vec y)$ such that 
$F^{\I_\M(do^\M(a,s))}=\phi^{\I_\M(s)}_{\vec y/\vec o}$, and hence $F^{\I_\M(do^\M(a,s))}$ is effectively computable.
\end{theorem}
\begin{proof}
Let 
$F(\vec x,do(a,s))\equiv\phi_F(\vec x,a,s)$ be the successor-state 
axiom for fluent $F$. 
For the extension of $F$ at situation $s'=do^\M(a,s)$,
we have that $\tup{\vec p,s'}\in F$ iff 
$\M,\vfo\models\phi_F(\vec x,A(\vec y),s)$,
for some $\vfo$ such that $\vfo(\vec x)=\vec p$
and $\vfo(\vec y)=\vec o$. Notice that $\phi_F$
contains, in general, action and situation terms,
and is uniform in $s$.
However,
by Theorem~\ref{prop:sitcalc-fo-act}, it
can be rewritten as
an equivalent action-term-free formula $\phi^A_F(\vec x,\vec y,s)$. 
Then, by suppressing the situation argument, we obtain: 
$\vec p\in F^{\I_\M(s')}$ iff 
$\I_\M(s),\vfo\models\phi^A_F(\vec x,\vec y)$,
for some $\vfo$ such that $\vfo(\vec x)=\vec p$
and $\vfo(\vec y)=\vec o$. That is, for $\phi=\phi_F^A$,
$F^{\I_\M(s')}=\phi^{\I_\M(s)}_{\vec y/\vec o}$.
Thus, since by boundedness of $\D$, $F^{\I_\M(s')}$ is finite, 
Theorem~\ref{th:fin-comp} implies the thesis.
\end{proof}
This result implies that, given $\I_\M(s)$ and an action
$a=A^\M(\vec o)$, we can obtain the interpretation of each $F$ at
$do^\M(a,s)$ by simply ``querying'' $\I_\M(s)$. Hence, by taking the
same interpretation of constants as in $\M$, we can construct
$\I_\M(do^\M(a,s))$, from $\I_\M(s)$ and the successor-state axioms
of $\D$.

\subsection{$\mymu$ over Transition Systems}
\label{sec:mu-calc-ts}

The results presented in Section~\ref{sec:action-suppression} and
\ref{sec:situation-suppression} suggest that, for the purpose of
verification of $\mymu$ formulas, one can operate on simpler structures
 than the models of situation calculus action theories. Indeed, as we saw,
both actions and situations can be essentially disregarded.
In this section, we introduce such simpler structures, namely
\emph{transition systems} (TS),
show  how $\mymu$ formulas are evaluated over them, and 
present some important results that allow us to perform the verification 
on TSs instead of on the original model.
The connection between models of situation calculus theories 
and transition systems will be discussed in Section~\ref{sec:ind-ts}.
By Theorem~\ref{prop:sitcalc-fo-act}, we can focus, without loss of generality,
on a variant of $\mymu$ where action terms 
do not occur.

By $Int^{\F,C}_\Delta$, we denote the set of all possible interpretations of the situation suppressed fluents 
in $\F$ and  constants in $C$, over the object domain $\Delta$.
A {\em transition system} (TS) (over the situation-suppressed fluents
$\F$, constants $C$, and object domain $\Delta$) is a tuple
$T=\tup{\Delta,Q,q_0,\ra,\I}$, where:
\begin{itemize}\itemsep=0pt
	\item $\Delta$ is the {\em object domain};
	\item $Q$ is the {\em set of states};
	\item $q_0\in Q$ is the {\em initial state};
	\item $\ra\subseteq Q\times Q$ is the {\em transition relation}; and
	\item $\I:Q\mapsto Int^{\F,C}_\Delta$ is the {\em labeling function} associating each state $q$ with an 
	interpretation $\I(q)=\tup{\Delta,\cdot^{\I(q)}}$ such that the constants in $C$ are 
	interpreted in the same way in all the states over which $\I$ is defined.
\end{itemize}

To interpret a $\mymu$ formula over a TS
$T=\tup{\Delta,Q,q_0,\ra,\I}$, we use valuations $(\vfo, \vso)$ formed
by an individual variable valuation $\vfo$ and a parametrized
predicate variable valuation $\vso$, as in
Section~\ref{sec:mucalc}.
We define the \emph{extension function}
\label{def:modat}
$\MODAT{\cdot}$, which maps $\mymu$ formulas to subsets of $Q$, as follows:
\[\begin{array}{lcl}
\MODAT{\varphi} &= & \{q \in Q\mid \I(q),v\models \varphi \}\\
\MODAT{\lnot \Phi} &= & Q - \MODAT{\Phi} \\
\MODAT{\Phi_1 \land \Phi_2} &= &\MODAT{\Phi_1}\cap\MODAT{\Phi_2}\\
\MODAT{\exists x.\,\inadom(x)\land\Phi} &=&
      \{q \in Q\mid\exists d\in \adom(\I(q)).\, q \in \MODATX{\Phi}{[x/d]}\}\\
\MODAT{\inadom(\vec{x})\land \DIAM{\Phi}} &=&
      \{q\in  Q\mid \vec{x}/\vec{d} \in \vfo \mbox{ and } \vec{d}\subseteq\adom(\I(q)) \mbox{ and }\\
      & & \quad\exists q'.  q\ra q'\mbox{ and } q' \in \MODAT{\Phi}\} \\
\MODAT{\inadom(\vec{x})\land \BOX{\Phi}} &=&
      \{q\in Q \mid \vec{x}/\vec{d} \in \vfo \mbox{ and } \vec{d}\subseteq\adom(\I(q)) \mbox{ and }\\
    & & \quad\forall q' .\, q \ra q' \mbox{ implies } q'\in \MODAT{\Phi}\} \\
\MODAT{Z} &= & \vso(Z)\\
\MODAT{\mu Z.\Phi} &= & \bigcap\{\E\subseteq Q \mid \MODATZ{\Phi}{[Z/\E]}
\subseteq \E \}
\end{array}
\]

Given a $\mymu$ formula $\Phi$, we say that a transition system $T$ \emph{satisfies
$\Phi$ at state $q$, under $v$ and $V$}, written $T,q,(\vfo,\vso)\models\Phi$, if 
$q\in \MODAT{\Phi}$. When $\Phi$ is closed on predicate variables, 
we omit $\vso$, as irrelevant, and  write $T,q,\vfo\models\Phi$. If $\Phi$ is closed on both individual and predicate variables we  simply write $T,q\models\Phi$. For closed formulas,
we say that \emph{$T$ satisfies $\Phi$}, written $T\models\Phi$, if $T,q_0\models\Phi$.

For our TSs we can prove a suitable version of 
the classical \emph{bisimulation invariance
results} for the $\mu$-calculus, which state that bisimilar TSs satisfy
exactly the same $\mu$-calculus formulas, see
e.g.,~\cite{BS07}. 
%
%
Obviously, the notion of bisimulation needed here is not the
classical one, but one that takes into account the FO interpretations
labeling the states of the transition systems, as well as the
controlled form of quantification across states allowed in $\mymu$.

We first recall the standard 
notions of \emph{isomorphism} and \emph{isomorphic interpretations}. 
Two FO interpretations
$\I_1=\tup{\Delta_1,\cdot^{\I_1}}$ and $\I_2=\tup{\Delta_2,\cdot^{\I_2}}$,
over the same fluents $\F$ and constants $C$,
are said to be \emph{isomorphic}, written $\I_1\sim\I_2$,
if there exists a bijection (called \emph{isomorphism}) 
$h:\Delta_1\mapsto \Delta_2$ such that:
\myi for every $F\in\F$, $\vec x\in F^{\I_1}$ if and only if $h(\vec x)\in F^{\I_2}$;
\myii for every $c\in C$, $c^{\I_2}=h(c^{\I_1})$. 
It is immediate to see that if $h$ is an isomorphism, then so is $h^{-1}$, and that 
$\sim$ is an equivalence relation.
Intuitively, for two interpretations to be isomorphic, it is required that one 
can be obtained from the other by renaming the individuals in the interpretation 
domain.
Notice that, necessarily, the interpretation domains of isomorphic interpretations
have same cardinality.
When needed, to make it explicit that $h$ is an isomorphism between $\I_1$ and
$\I_2$, we write $\I_1\sim_h \I_2$.
We denote by $\restrict{h}{D_1}$ the \emph{restriction} of $h$ to $D_1$, i.e., the mapping
$\restrict{h}{D_1}:D_1\mapsto h(D_1)$, such that
$\restrict{h}{D_1}(d)=h(d)$, for every $d\in D_1$.
In addition, recall that $\arestrict{\I}=\tup{\adom(\I),\cdot^{\I}}$ denotes the \emph{restriction of an interpretation} $\I=\tup{\Delta,\cdot^\I}$
to its active domain.

The bisimulation relation that captures \mymu  can be defined as follows.  
Let $T_1 = \tup{\Delta_1,Q_1,q_{10},\ra_1,\I_1}$ and
$T_2= \tup{\Delta_2, Q_2, q_{20}, \ra_2, \I_2}$ 
be two transition systems
(over the situation-suppressed fluents and constants of an action theory $\D$),
and let $H$ be the set of all possible bijections $h:D_1\mapsto D_2$,
for $D_1\subseteq \Delta_1$ and $D_2\subseteq \Delta_2$.
A relation $B\subseteq Q_1\times H\times Q_2$ is a
\emph{persistence-preserving bisimulation}\label{def:bisim}
between $T_1$ and $T_2$,
if $\tup{q_1,h,q_2}\in B$ implies that:
\begin{enumerate}
\item\label{bisim:req1} 
	$\arestrict{\I_1}(q_1)\sim_h\arestrict{\I_2}(q_2)$
;

\item\label{bisim:req2} for each $q_1'\in Q_1$, if $q_1 \ra_1 q_1'$ then there exists
  $q_2'\in Q_2$ such that:
\begin{enumerate}
\item  \label{bisim:req2a}$q_2 \ra_2 q_2'$,
\item  \label{bisim:req2b}there exists a bijection 
  $h':\adom(\I_1(q_1))\cup \adom(\I_1(q'_1))\mapsto\adom(\I_2(q_2))\cup \adom(\I_2(q'_2))$ such that its restriction $\restrict{h'}{\adom(\I_1(q_1))}$ 
  coincides with $h$ and its restriction $\restrict{h'}{\adom(\I_1(q'_1))}$ 
  is such that
  $\tup{q_1',\restrict{h'}{\adom(\I_1(q'_1))},q_2'}\in B$;

\end{enumerate}  

\item\label{bisim:req3} for each $q_2'\in Q_2$, if $q_2 \ra_2 q_2'$ then there exists
 $q_1'\in Q_1$ such that:
\begin{enumerate}
\item  \label{bisim:req3a}$q_1 \ra_2 q_1'$, 
\item \label{bisim:req3b}there exists a bijection 
  $h':\adom(\I_1(q_1))\cup \adom(\I_1(q'_1))\mapsto\adom(\I_2(q_2))\cup \adom(\I_2(q'_2))$ such that its restriction $\restrict{h'}{\adom(\I_1(q_1))}$ 
  coincides with $h$ and its restriction $\restrict{h'}{\adom(\I_1(q'_1))}$ 
  is such that
  $\tup{q_1',\restrict{h'}{\adom(\I_1(q'_1))},q_2'}\in B$.
\end{enumerate}
\end{enumerate}
Notice that requirements~\ref{bisim:req2b} and~\ref{bisim:req3b}  
impose the existence of a bijection $h'$ that preserves the bijection $h$ 
(in fact, the isomorphism) between the objects in $\adom(\I_1(q_1))$ and those in
$\adom(\I_2(q_2))$; this essentially means that the ``identity'' of such 
objects is preserved along the transition. Moreover, $h'$ is required to 
induce an isomorphism between $\adom(\I_1(q'_1))$ and $\adom(\I_2(q'_2))$, 
when restricted to $\adom(\I_1(q'_1))$, such that 
$\tup{q_1',\restrict{h'}{\adom(\I_1(q'_1))},q_2'}\in B$.

We say that a state $q_1 \in Q_1$ is \emph{(persistence-preserving)
  bisimilar} to $q_2 \in Q_2$, written $q_1 \approx q_2$, if there
exists a persistence-preserving bisimulation $B$ between $T_1$ and
$T_2$ such that $\tup{q_1,h,q_2}\in B$, for some $h$;
when needed, we also write $q_1 \approx_h q_2$, 
to explicitly name $h$.
Finally, a transition system
$T_1$ is said to be \emph{persistence-preserving bisimilar} to $T_2$,
written $T_1 \approx T_2$, if $q_{10} \approx q_{20}$. It is immediate
to see that bisimilarity between states and transition systems, i.e.,
the (overloaded) relation $\approx$, is an equivalence relation.

Next, we prove a result (Theorem~\ref{thm:mymu-bisimulation})
saying that \mymu enjoys invariance under this
notion of bisimulation. To this end, 
we first show the result for the simpler logic $\myhm$, obtained
from $\mymu$ by dropping the fixpoint
construct. Namely,  $\myhm$ is defined as:
\[\begin{array}{l}
  \Phi ::=  \varphi  \mid \lnot \Phi \mid \Phi_1 \land \Phi_2 \mid \exists
    x. \inadom(x) \land \Phi \mid {} \inadom(\vec{x}) \land \DIAM{\Phi}  \mid
   \inadom(\vec{x}) \land\BOX{ \Phi}
\end{array}
\]
Such a logic corresponds to a first-order variant of the
Hennessy-Milner Logic \cite{HM80}. Note that its semantics is
completely independent from the second-order valuation. 

Given an individual variable valuation $\vfo$ we denote by
$\image{\vfo}$ its image on the object domain.
\begin{lemma}\label{lemma:myhm-bisimulation}
  Consider two transition systems
  $T_1= \tup{\Delta_1,Q_1,q_{10},\ra_1,\I_1}$ and
  $T_2= \tup{\Delta_2, Q_2, q_{20}, \ra_2, \I_2}$, 
  two states $q_1\in Q_1$, $q_2 \in Q_2$, such that
  $q_1 \approx_h q_2$,
 and two individual variable valuations $\vfo_1$ and
  $\vfo_2$ mapping variables to $\Delta_1$ and $\Delta_2$, respectively. 
  If there exists a bijection $\hat h$ between $\adom(\I_1(q_1))\cup\image{\vfo_1}$ and  $\adom(\I_2(q_2))\cup\image{\vfo_2}$ whose restriction 
  $\restrict{\hat h}{\adom(\I_1(q_1))}$ coincides with $h$ and 
  such that for each individual variable $x$, $\hat h(\vfo_1(x))= \vfo_2(x)$, 
  then for every formula $\Phi$ of $\myhm$, possibly open on individual 
  variables, we have that:
\[T_1,q_1,\vfo_1 \models \Phi \, \textrm{ if and only if }\,  T_2,q_2,\vfo_2 \models \Phi. \]
\end{lemma}
\begin{proof}
We proceed by induction on the structure of $\Phi$. 
For $\Phi=\varphi$, we observe that, by Theorem~\ref{thm:dom-ind},
$\I_i(q_i),\vfo_i\models\varphi$ if and only if
$\arestrict{\I_i}(q_i),\vfo_i\models\varphi'$ ($i=1,2$), for
$\varphi'$ the rewriting of $\varphi$ as its domain-independent
version.  Further, since
$\arestrict{\I_1}(q_1)\sim_h\arestrict{\I_2}(q_2)$, and there is a
bijection $\hat h$ between the objects assigned to variables by $\vfo_1$ and
$\vfo_2$ (even if they are not in $\adom(\I_1(q_1))$ or
$\adom(\I_2(q_2))$), by the invariance of FOL wrt isomorphic
interpretations, it follows that
$\arestrict{\I_1}(q_1),\vfo_1\models\varphi'$ if and only if
$\arestrict{\I_2}(q_2),\vfo_2\models\varphi'$.  These two facts easily
imply the thesis.
The cases of boolean connectives are obtained by
straightforward induction using the same individual valuations
$\vfo_1$ and $\vfo_2$ and the same bijection $\hat h$.

\smallskip For $\Phi=\exists y.\inadom(y)\land\Phi'$. Suppose that
$T_1,q_1,\vfo_1\models \Phi$. Then, for some $d_1$, it is the case that
$T_1,q_1,\vfo_1[y/d_1]\models \inadom(y)\land\Phi'$. Notice that this
implies $d_1\in\adom(\I_1(q_1))$, then $\hat h(d_1)=h(d_1)=d_2$, for
some $d_2\in \adom(\I_2(q_2))$, as $\hat h$ coincides with $h$ on 
$\adom(\I_1(q_1))$. 
Consider the individual valuation
$\vfo_2[y/d_2]$. For every variable $x$ we have
$\hat h(\vfo_1[y/d_1](x))=\vfo_2[y/d_2](x)$ (for $y$ we have
$\vfo_2[y/d_2](y)=d_2=\hat h(d_1)=\hat h(\vfo_1[y/d_1](y))$). 
Hence, using these new valuations
and the same bijection $\hat h$, now restricted to 
$\image{\vfo_1[y/d_1]}$ and $\image{\vfo_2[y/d_2]}$ (to take into account
the assignments to $y$), we can apply the induction hypothesis, and
conclude that $T_2,q_2,\vfo_2[y/d_2]\models \inadom(y)\land\Phi'$,
which implies $T_2,q_2,\vfo_2\models \Phi$.  The other direction is
proven symmetrically.

\smallskip For $\Phi=\inadom(\vec{x}) \land \DIAM{\Phi'}$. Suppose
that $T_1,q_1,\vfo_1\models (\inadom(\vec{x}) \land \DIAM{\Phi'})$.
By definition, this implies that $\vfo_1(x_i) \in \adom(\I_1(q_1))$
for each $x_i\in\vec{x}$, and there exists a transition
$q_1 \ra_1 q'_1$ such that $T_1,q'_1,\vfo_1 \models \Phi'$.
Since $q_1 \approx_h q_2$, there exist: 
\myi a transition $q_2 \ra_2 q'_2$, and
\myii a bijection
$h':\adom(\I_1(q_1))\cup \adom(\I_1(q'_1))\mapsto\adom(\I_2(q_2))\cup
\adom(\I_2(q'_2))$
such that its restriction $\restrict{h'}{\adom(\I_1(q_1))}$ 
coincides with $h$,
its restriction
$\restrict{h'}{\adom(\I_1(q'_1))}$ 
is an isomorphism such that 
$\arestrict{\I_1}(q'_1)\sim_{\restrict{h'}{\adom(\I_1(q'_1))}}\arestrict{\I_2}(q'_2)$,
and $q'_1\approx_{\restrict{h'}{\adom(\I_1(q'_1))}}q_2'$.
Now consider two new variable valuations $\vfo'_1$ and $\vfo'_2$, defined as follows:
\begin{itemize}
\item for $x_i\in \vec{x}$ (for which we have that
  $\vfo_1(x_i)\in \adom(\I_1(q_1))$), let 
  $\vfo'_1(x_i)=\vfo_1(x_i)$ and $\vfo'_2(x_i)=\vfo_2(x_i)$;
\item choose $d_1 \in \Delta_1$ and, for all $y\not \in \vec{x}$, 
let $\vfo'_1(y)=d_1$, then: 
if $d_1 \in \adom(\I_1(q_1))\cup\adom(\I_1(q'_1))$, 
for all $y\notin\vec x$, let
  $\vfo'_2(y)=h'(d_1)$; else, 
  choose $d_2 \not\in \adom(\I_2(q_2))\cup\adom(\I_2(q'_2))$,
  let, for all $y\notin\vec x$, $\vfo'_2(y)=d_2$, and contextually extend $h'$ so that $h'(d_1)=d_2$.
\end{itemize}
As a result, for all variables $x$, we have $h'(\vfo'_1(x))=\vfo'_2(x)$ 
(for $h'$ possibly extended as above).
Consider the bijection $\hat h'=\restrict{h'}{\adom(\I_1(q'_1))\cup\image{\vfo'_1}}$.
With this new bijection and the valuations
$\vfo'_1$ and $\vfo'_2$, we can apply the induction hypothesis,
and obtain that $T_1,q'_1,\vfo_1 \models \Phi'$ implies
$T_2,q'_2,\vfo'_2\models \Phi'$, and since $q_2 \ra_2 q'_2$, we have
that $T_2,q_2,\vfo'_2 \models (\inadom(\vec{x}) \land \DIAM{\Phi'})$.
Now, observe that the only free variables of
$(\inadom(\vec{x}) \land \DIAM{\Phi'})$ are $x_i\in \vec{x}$,
and that, for these, we have $\vfo'_1(x_i)=\vfo_1(x_i)$ and 
$\vfo'_2(x_i)=\vfo_2(x_i)$. Therefore, we can conclude that
$T_2,q_2,\vfo_2 \models (\inadom(\vec{x}) \land \DIAM{\Phi'})$.
The other direction can be proven in a symmetric way.

\smallskip For $\Phi=\inadom(\vec{x}) \land \BOX{ \Phi''}$: we observe that we can
rewrite $\Phi$ as 
$\lnot(\inadom(\vec{x}) \limp \DIAM{ \Phi'})$, with
$\Phi' = \neg \Phi''$.
Then, assume that
$T_1,q_1,\vfo_1 \models (\inadom(\vec{x}) \limp\DIAM{ \Phi'})$.  By
definition, this implies that: \myi either for some $x_i\in\vec{x}$ we
have $\vfo_1(x_i)\not\in \adom(\I_1(q_1))$; or \myii for all
$x_i\in\vec{x}$ we have $\vfo_1(x_i)\in \adom(\I_1(q_1))$ and there
exists a transition $q_1 \ra_1 q'_1$ such that
$T_1,q'_1,\vfo_1 \models \Phi'$.  We distinguish the two cases:
\begin{itemize}
\item If for some $x_i\in\vec{x}$,
  $\vfo_1(x_i)\not\in \adom(\I_1(q_1))$, then we have that
  $\vfo_2(x_i)\not\in \adom(\I_2(q_2))$. Indeed, assume toward
  contradiction that $\vfo_2(x_i) \in\adom(\I_2(q_2))$.  Since
  $\arestrict{\I_1}(q_1) \sim_h \arestrict{\I_2}(q_2)$ it follows that
  the inverse $h^{-1} $ of $h$ is unique, hence
  $h^{-1}(\vfo_2(x_i))=\vfo_1(x_i)$ and
  $\vfo_1(x_i) \in \adom(\I_1(q_1))$, getting a contradiction.  Thus, we
  have that $T_2,q_2,\vfo_2 \not \models \inadom(\vec{x})$ and so
  $T_2,q_2,\vfo_2 \models (\inadom(\vec{x}) \limp \DIAM{\Phi'})$.

\item If for all $x_i\in\vec{x}$, $\vfo_i(x_i)\in \adom(\I_1(q_1))$,
  we can proceed in the same way as for the case of
  $\Phi=\inadom(\vec{x}) \land \DIAM{\Phi'}$.
\end{itemize}
The other direction is proven symmetrically.
\end{proof}

We can now extend the result to the whole $\mymu$.
\begin{lemma}\label{lemma:mymu-bisimulation}
 Consider two transition systems
  $T_1= \tup{\Delta_1,Q_1,q_{10},\ra_1,\I_1}$ and
  $T_2= \tup{\Delta_2, Q_2, q_{20}, \ra_2, \I_2}$, 
  two states $q_1\in Q_1$, $q_2 \in Q_2$, such that
  $q_1 \approx_h q_2$,
 and two individual variable valuations $\vfo_1$ and
  $\vfo_2$ mapping variables to $\Delta_1$ and $\Delta_2$, respectively. 
  If there exists a bijection $\hat h$ between $\adom(\I_1(q_1))\cup\image{\vfo_1}$ and  $\adom(\I(_2q_2))\cup\image{\vfo_2}$ whose restriction 
  $\restrict{\hat h}{\adom(\I_1(q_1))}$ coincides with $h$ and 
  such that for each individual variable $x$, $\hat h(\vfo_1(x))= \vfo_2(x)$, 
  then for every formula $\Phi$ of $\mymu$,
  closed on the predicate variables but possibly open on the
  individual variables, we have:
\[T_1,q_1,\vfo_1 \models \Phi \, \textrm{ if and only if }\,  T_2,q_2,\vfo_2 \models \Phi. \]
\end{lemma}
\begin{proof}
  We prove the theorem in two steps. First, we show that Lemma
  \ref{lemma:myhm-bisimulation} can be extended to the infinitary
  version of $\myhm$ that supports arbitrary infinite disjunction of
  formulas sharing the same free variables \cite{vanBenthem83}. Then, we recall that
  fixpoints can be translated into this infinitary logic, thus
  guaranteeing invariance for the whole $\mymu$ logic.
Let $\Psi$ be a possibly infinite set of open $\myhm$ formulas.
Given a transition system
  $T = \tup{\Delta,Q,q_0,\ra,\I}$, the semantics of $\bigvee \Psi$ is
  $\MODAT{\bigvee \Psi} = \bigcup_{\psi \in \Psi}
  \MODAT{\psi}$.
  Therefore, given a state $q$ of $T$ and a variable valuation $\vfo$,
  we have $T,q,\vfo \models \Psi$ if and only if
  $T,q,\vfo \models \psi$ for some $\psi \in \Psi$. Arbitrary infinite
  conjunction is obtained for free through negation.
  Lemma~\ref{lemma:myhm-bisimulation} extends to this arbitrary
  infinite disjunction.  By the induction hypothesis, under the
  assumption of the Lemma, we can assume that for every formula
  $\psi \in \Psi$, we have $T_1,q_{10},\vfo_1 \models \psi$ if and
  only if $T_2,q_{20},\vfo_2 \models \psi$.  Given the semantics of
  $\bigvee \Psi$ above, this implies that
  $T_1,q_{10},\vfo_1 \models \bigvee \Psi$ if and only if
  $T_2,q_{20},\vfo_2 \models \bigvee \Psi$.

  In order to extend the result to the whole $\mymu$, we translate
  $\mu$-calculus \emph{approximates} into the infinitary  $\myhm$  by  (see~\cite{BS07,vanBenthem83}),
  where the approximant of index $\alpha$ is denoted by
  $\mu^\alpha Z.\Phi$ for least fixpoint formulas $\mu Z.\Phi$ and
  $\nu^\alpha Z.\Phi$ for greatest fixpoint formulas $\nu
  Z.\Phi$.
  This is a standard result that holds also for $\mymu$. In
  particular, such approximates are as follows:
\[\begin{array}{rcl@{\qquad}rcl}
  \mu^0 Z.\Phi & = & \false 
  &  \nu^0 Z.\Phi & =& \true\\
  \mu^{\beta+1} Z.\Phi & =& \Phi[Z/\mu^\beta Z.\Phi]
  & \nu^{\beta+1} Z.\Phi & =& \Phi[Z/\nu^\beta Z.\Phi]\\
  \mu^\lambda Z.\Phi & =& \bigvee_{\beta < \lambda} \mu^\beta Z. \Phi 
  &
  \nu^\lambda Z.\Phi & =& \bigwedge_{\beta < \lambda} \nu^\beta Z. \Phi
\end{array}
\]
where $\lambda$ is a limit ordinal, and the notation $\Phi[Z/\nu^\beta Z.\Phi]$ denotes the formula obtained from $\Phi$ by replacing each occurrence of $Z$ by $\nu^\beta Z.\Phi$.
By Tarski and Knaster Theorem~\cite{Tarski55}, the fixpoints and their
approximates are connected by the following properties: given a
transition system $T$ and a state $q$ of $T$,
\begin{itemize}
\item $q \in \MODAT{\mu Z.\Phi}$ if and only if there exists an
  ordinal $\alpha$ such that $s \in \MODAT{\mu^\alpha Z.\Phi}$ and, for
  every $\beta < \alpha$,  it holds that $s \not\in \MODA{\mu^\beta Z.\Phi}$;
\item $q \not\in \MODAT{\nu Z.\Phi}$ if and only if there exists an ordinal $\alpha$ such that $s \not\in \MODAT{\nu^\alpha
    Z.\Phi}$ and, for every $\beta < \alpha$, it holds that $q \in \MODA{\nu^\beta
    Z.\Phi}$.
\end{itemize}
Since each approximate, including the ones corresponding exactly to the least and greatest fixpoints, can be written as an infinitary $\myhm$ formula,  we get the thesis.
\end{proof}

With this lemma in place we can prove the invariance result. 

\begin{theorem}
\label{thm:mymu-bisimulation}
Consider two transition systems $T_1=\tup{\Delta_1,Q_1,q_{10},\ra_1,\I_1}$ 
and $T_2=\tup{\Delta_2,Q_2,q_{20},\ra_2,\I_2}$.
If $T_1 \approx T_2$,
then, for every $\mymu$ closed formula $\Phi$
\[T_1 \models \Phi \textrm{ if and only if } T_2 \models \Phi.\]
\end{theorem}
\begin{proof}
  If $T_1 \approx T_2$ then
  for some bijection $h$ we have $q_{10} \approx_h q_{20}$. This
  implies that
  $\arestrict{\I_1}(q_{10})\sim_h\arestrict{\I_2}(q_{20})$.  Now
  consider the variable valuations $\vfo_1$ and $\vfo_2$ defined as
  follows (notice that since $\Phi$ is closed such individual
  valuations are irrelevant in evaluating it): 
  choose an arbitrary $d_1 \in \Delta_1$ and let, 
  for all variables $x$,
  $\vfo_1(x)=d_1$; if $d_1 \in \adom(\I_1(q_1))$, 
  let, for all $x$, $\vfo_2(x)=h(d_1)$;
  else, choose $d_2 \not\in \adom(\I_2(q_2))$ 
  and let, for all $x$, $\vfo'_2(x)=d_2$.

  Now, define a bijection $h'$ such that for all $d\in\adom(\I(q_1))$,
  $h'(d)=h(d)$, and if $d_1 \not\in \adom(\I_1(q_1))$,  
  $h'(d_1)=d_2$.
  It can be seen that  $h'$ is a bijection between
  $\adom(\I_1(q_1)\cup \image{\vfo_1}$ and
  $\adom(\I_2(q_2)\cup \image{\vfo_2}$ such that
  $\arestrict{\I_1}(q_1) \sim_{\restrict{h'}{\adom(\I_1(q_1))}}
  \arestrict{\I_2}(q_2)$
  and for all variables $x$, $h'(\vfo_1(x))=\vfo_2(x)$. 
  Hence, by Lemma~\ref{lemma:mymu-bisimulation}, we get the thesis.
\end{proof}

Thus, to check whether a transition system $T$ satisfies a $\mymu$
formula $\Phi$, one can perform the check on any transition system
$T'$ that is bisimilar to $T$. This is particularly useful in those
cases where $T$ is infinite-state but admits some finite-state
bisimilar transition system. We exploit this result later on.


\subsection{Transition Systems Induced by a Situation Calculus Theory}\label{sec:ind-ts}

Among the various TSs, we are interested in those \emph{induced} by models of
the situation calculus action theory $\D$.
Consider a model $\M$ of $\D$ with object domain $\Delta$ and situation domain $\S$.
The TS {\em induced by $\M$}\label{def:induced-ts}
is the labelled TS \mbox{$T_{\M}=\tup{\Delta,Q,q_0,\I,\ra}$}, such that:
\begin{itemize}\itemsep=0pt
	\item $Q=\S$ is the set of {\em possible states}, 
		each corresponding to a distinct executable situation in $\S$;
	\item $q_0=S_0^\M\in Q$ is the {\em initial state}, 
		with $S_0^\M$ the initial situation of $\D$;
	\item ${\ra}\subseteq Q\times Q$ is the {\em transition relation}
		such that~$q\ra q'$ iff there exists some action $a$ such 
		that $\tup{a,q}\in\Poss^{\M}$ and $q'=do^{\M}(a,q)$.
	\item $\I: Q\mapsto Int^{\F,C}_\Delta$ 
		is the {\em labeling function} 
		associating each state (situation) $q$ with the interpretation 
		$\I(q)=\I_\M(q)$.
\end{itemize}
As it can be seen, the TS induced by a model $\M$ is essentially the 
tree of executable situations, with each situation labelled by an 
interpretation of fluents (and constants), corresponding to the interpretation 
associated by $\M$ to that situation. Notice that transitions do not carry any 
information about the corresponding triggering action.

We can now show that the semantics of $\mymu$ on a model can alternatively 
be given in terms of the corresponding induced TS. 

\begin{theorem}\label{th:InducedTS}
  Let $\D$ be an action theory, $\M$  a model of $\D$ with (infinite) object domain $\Delta$ and
  situation domain $\S$, and $T_{\M}$ the corresponding induced TS.
  Then for every $\mymu$ formula $\Phi$ (with no occurrence of action terms) we have that:
\[\MODA{\Phi}=\MODATM{\Phi}\] 
\end{theorem}
\begin{proof}
  By induction on the structure of $\Phi$. For the base case
   of an open uniform situation-suppressed situation calculus  formula $\varphi$, 
   we need to prove that
\[\MODA{\varphi} =  \{s \in \S\mid \M,\vfo\models \varphi[s] \}  \;=\;
\MODATM{\varphi}  =  \{s \in \S\mid \I(s),\vfo\models \varphi\}.\]
This is indeed the case: since no action terms occur in $\varphi$
	and $\varphi$ is uniform in $s$, 
	the evaluation of $\varphi$ depends only on the interpretation of each fluent (and constant) at $s$, 
	i.e., on $\I_\M(s)$.
Once this base case is settled, the inductive cases are straightforward.
\end{proof}

\subsection{Abstract Finite-State Transition System} \label{sec:finite-abstraction}

As shown above, satisfaction of $\mymu$ formulas is 
preserved by persistence-preserving bisimulations.
This holds even between an infinite- and a finite-state TS.
When this is the case, the verification can be performed 
on the finite TS using standard $\mu$-calculus model checking techniques, 
which essentially perform fixpoint computations on a finite state space. 
We next show how, for the case of bounded theories, one can construct a finite TS $T_F$ that is
bisimilar to the TS $T_{\M}$ induced by  $\M$.

\renewcommand{\algorithmicrequire}{\textbf{Input:}}
\renewcommand{\algorithmicensure}{\textbf{Output:}}
\newcommand{\PICK}{\textbf{pick \xspace}}
\begin{algorithm}
\caption{Computation of a finite-state TS persistence-preserving bisimilar to $T_{\M}$.\label{alg:tf}}
\begin{algorithmic}[1]
	\REQUIRE A  basic action theory $\D$ bounded by $b$, with complete information on $S_0$, and a model $\M$ of $\D$ with infinite object domain $\Delta$
	\ENSURE A finite-state TS $T_F=\tup{\Delta,Q,q_0,\I,\ra}$ persistence-preserving bisimilar to $T_{\M}$
	\STATE \label{alg:init}
		\LET 
			$\F$ the set of fluents of $\D$,
			$C$ the set of constants explicitly mentioned in $\D$;
	\STATE \label{alg:init2}
		\LET $Q:=\set{q_0}$, for $q_0$ a fresh state;
	\STATE \label{alg:init3}
		\LET $\I(q_0)=\I_\M(S^\M_0)$;
	\STATE \label{alg:init4}
		\LET ${\ra}:=\emptyset$;\\
	\STATE \label{alg:init5}
		\LET $Q_{te}:=\set{q_0}$;\\
       \WHILE{($Q_{te}\neq \emptyset$)}\label{alg:while}
        \STATE\label{alg:pick} 
                       \PICK $q\in Q_{te}$; 
        \STATE\label{alg:pop}
                       \LET $Q_{te} := Q_{te} - \set{q}$;
	\STATE \label{alg:choose-O}
		\LET $O\subseteq \Delta$ be any (finite) 
			set of objects such that:\\
				\quad \myi $\card{O}=\max\set{\card{\vec x}\mid A(\vec x)\in\A}$;\\
				\quad \myii $O\cap \adom(\I(q))=\emptyset$;\\
				\quad \myiii $\card{O\cap \bigcup_{q\in Q}\adom(\I(q))}$ is maximal (subject to \myi and \myii).
	\FORALL{action types $A(\vec x)$ of $\D$}\label{alg:forall-types}
		\FORALL{\label{alg:poss}valuations $\vfo$
			such that $\vfo(\vec x) \in (\adom(\I(q))\cup O)^{\card{\vec x}}$ and\\
		\qquad\qquad$\I(q),v\models Poss(A(\vec x))$}
			\STATE\label{alg:ssa}
				\LET $\I'=\tup{\Delta,\cdot^{\I'}}$ be an interpretation such that:
						\myi $c^\I=c^\M$, for all constants in $C$;
						\myii $F^{\I'}=\set{\vec d \mid \I(q),v[\vec y/\vec d]\models \phi_F(A(\vec x),\vec y)}$, for 
						$\phi_F(a,\vec y)$ the (situation-suppressed) RHS of the SSA 
						of fluent $F$.
			\IF {(there exists $q'\in Q$ and an isomorphism 
				$h$ between 
				$\I'$ and  $\I(q')$  that is the identity on 
                            $\adom(\I(q))$)} \label{alg:if-add}
				\STATE ${\ra} := {\ra}\cup\set{q\ra q'}$; \label{alg:if-add-trans}
				\ELSE\label{alg:else-add}
					\STATE \label{alg:add}
						\LET
							$Q:=Q\uplus\set{q'}$, for $q'$ a fresh state;\\
							\quad $\I(q'):=\I'$;\\
							\quad ${\ra} :=  {\ra}\cup\set{q\ra q'}$;\\
                                                		\quad $Q_{te}:=Q_{te}\uplus\set{q'}$;
			\ENDIF
		\ENDFOR
	\ENDFOR\label{alg:end-forall-types}
      \ENDWHILE\label{alg:end-while}
      \RETURN $T_F=\tup{\Delta,Q,q_0,\I,\ra}$
\end{algorithmic}
\end{algorithm}
We construct $T_F$ using Procedure~\ref{alg:tf}.
The procedure takes as input an action theory $\D$ (with complete information on 
the initial situation) bounded by $b$ and a model $\M$ of $\D$ with infinite object domain $\Delta$,\footnote{In fact, given the object domain $\Delta$, the model $\M$ is fully determined by $\D$ modulo object renaming.}
and returns a finite-state TS $T_F$ bisimilar  to $T_{\M}$.
$T_F$ is built incrementally,
through iterative refinements of the set of states $Q$,
the interpretation function $\I$, and the transition relation $\ra$.
Initially, $Q$ contains only the initial state $q_0$ (line~\ref{alg:init2});
$\I(q_0)$ interprets constants and fluents in the same way as $\M$ at 
the initial situation (line~\ref{alg:init3});
and $\ra$ is empty (line~\ref{alg:init4}).
The set $Q_{te}$ contains the states of $T_F$ to be ``expanded''
(initially $q_0$ only, line~\ref{alg:init5}); this is done at each iteration of the \textbf{while} loop 
(lines~\ref{alg:while}--\ref{alg:end-while}), as explained next.

Firstly, a state $q$ is extracted from $Q_{te}$ (lines~\ref{alg:pick} and \ref{alg:pop}).
Then, a finite subset $O$ of objects from $\Delta$ is defined (line~\ref{alg:choose-O}).
The values from $O$, together with those from $\adom(\I(q))$, 
are used, in combination with the action types, to generate actions executable 
on the interpretation $\I(q)$\footnote{Notice that since $\Poss(a,s)$ is 
uniform in $s$, the situation does not play any role in establishing whether,
for given $a$ and $s$, $\Poss(a,s)$ holds. 
In fact, only the interpretation of fluents (and constants)
at $s$ matters. Consequently, one can take such an interpretation and safely 
suppress the situation argument.} (lines~\ref{alg:forall-types},\ref{alg:poss}).
The particular choice of $O$ guarantees that the set of generated 
actions, while finite, is fully representative,
for the purpose of verification,
of all the (possibly infinitely many) actions executable on $\I(q)$ 
(see Theorem~\ref{thm:bisimulation-finite-induced}). Moreover, the 
objects are chosen so as to maximize \emph{reuse} of the objects occurring 
in the interpretation of the states already in $Q$.

The actual expansion step consists in computing, for each generated action,
the interpretation $\I'$ obtained by executing the action on (a situation with 
interpretation) $\I(q)$. This is done by computing,
on $\I(q)$, the answers 
to the right-hand side $\phi(a,\vec y)$ of the (situation-suppressed) successor state axiom
of each fluent $F$, with $a$ set to the current action  (line~\ref{alg:ssa}). 
Once $\I'$ has been computed, two cases are possible: either it is isomorphic to 
some interpretation $\I(q')$ labeling an existing state $q'\in Q$ (line~\ref{alg:if-add}), under 
some isomorphism that preserves $\I(q)$,
or it is not (line~\ref{alg:else-add}). In the former case, the transition relation
is simply updated with a transition from $q$ to $q'$ (line~\ref{alg:if-add-trans})
and no new state is generated. We stress that, in this case, the isomorphism is 
defined over the whole $\Delta$, not only over the active domains of
the interpretations.
In the latter case, a fresh state 
$q'$  with labeling $\I(q')$ is added to $Q$, and the transition relation is 
updated with $q\ra q'$ (lines~\ref{alg:add}). 
Further, $q'$ is also added to $Q_{te}$, so as to be expanded in future 
iterations.
The procedure iterates over the expansion step until the set $Q_{te}$ is empty,
i.e., unitl there are no more states to expand.

We observe that the choice of $q'$ at line~\ref{alg:if-add-trans}
guarantees the existence of an 
isomorphism $h'$ between $\I'$ and $\I(q')$ that is the identity
on $\adom(\I(q))$. That is, any object occurring in $\I'$ that comes 
from $\I(q)$ must be mapped into itself.
The purpose of this choice is to avoid adding a fresh state $q''$ (with interpretation $\I'$) 
to $Q$ but \emph{reuse} any state $q'$ already in $Q$, if
bisimilar to the
candidate $q''$. This is a key step for the procedure to construct
a transition system that is both finite and persistence-preserving bisimilar to $T_\M$.

We can now show that Procedure~\ref{alg:tf} terminates and 
returns a TS persistence-preserving bisimilar to $T_{\M}$.  This result is split into two
main results: Theorem~\ref{thm:termination}, which shows 
that the procedure terminates, returning a finite TS, 
and Theorem~\ref{thm:bisimulation-finite-induced}, which shows that the 
obtained TS is indeed persistence-preserving bisimilar to $T_{\M}$.

To prove termination, we first
derive a bound on the active domain of the interpretations 
labeling the states in $Q$.
\begin{lemma}\label{lem:qbound}
There exists a value $b'=\sum_{F\in\F}b\cdot a_F+\card{C}$ such that, at any iteration of Procedure~\ref{alg:tf}
and for any $q\in Q$,
$\card{\adom(\I(q))}\leq b'$, where $b$ is the value 
bounding $\D$, $a_F$ the arity of fluent $F$, and $C$ the set of constants explicitly mentioned in $\D$.
\end{lemma}
\begin{proof}
We first show that: $(\dagger)$ for every $q\in Q$,
there exists a situation $s$ executable in $\D$ 
such that $\I(q)=\I_\M(s)$.
This intuitively means that, modulo situation suppression, 
every state of $T_F$ is labelled
by an interpretation that matches that of $\M$ on
constants and fluents at some executable situation $s$.

The proof is by induction on $Q$. 
For $q_0$, the thesis follows 
by the definition of $\I(q_0)$ at line~\ref{alg:init3}, as $S_0^\M$ is executable.
For the induction step, consider $q\in Q$ and assume, by the induction hypothesis, 
that $\I(q)$ is as above, for an executable situation $s$. 
Then, for any valuation (of object variables) $\vfo$, we have that
$\I(q),\vfo\models \Poss(A(\vec x))$ if and only if 
$\I_\M(s),\vfo\models \Poss(A(\vec x))$, that is,
by Theorem~\ref{th:mod-int},
$\M,\vfo'\models\Poss(A(\vec x),\sigma)$,
for $\sigma$ a situation variable and $\vfo'$ a situation calculus 
variable assignment analogous to $\vfo$ on all individual variables 
and such that $\vfo'(\sigma)=s$.
Thus, by line~\ref{alg:poss}, $A(\vec x)$ is executable at $s$ (with
respect to $\M$ and $\vfo$). 
Similarly, for any fluent $F$ and valuation $\vfo$, we have that 
$\I(q),v \models\phi_F(A(\vec x),\vec y)$ iff 
$\M,\vfo' \models\phi_F(A(\vec x),\vec y,\sigma)$, that is, since 
$F(\vec y,do(a,\sigma))\equiv\phi_F(a,\vec y,\sigma)$ (by definition of 
successor-state axiom),
$\I(q),\vfo \models\phi_F(A(\vec x),\vec y)$ iff 
$\M,\vfo' \models F(\vec y,do(A(\vec x),\sigma))$.
But then, since by line~\ref{alg:ssa},
$F^{\I'}=\set{\vec d\in\vec \Delta\mid \I(q),\vfo[\vec y/\vec d] \models\phi_F(A(\vec x),\vec y)}$,
it follows that 
$\I',\vfo\models F(\vec y)$ iff $\M,\vfo'[\vec y/\vec d] \models F(\vec y,do(A(\vec x),\sigma))$. 
Thus, $F^{\I'}=\set{\vec d\in\Delta\mid \M,\vfo'[\vec y/\vec d] \models F(x,do(A(\vec x),\sigma))}$.
Therefore, when a state $q'$ is added to $Q$ (line~\ref{alg:add}), 
its labeling $\I(q')=\I'$ is such  that 
$\I(q')=\I_\M(do^\M(A^\M(v(\vec x)),s))$.
This proves $(\dagger)$. 

Observe that $(\dagger)$ and the boundedness of $\D$ imply, together, 
that $\card{\adom(\I(q))}$ is bounded, for any $q\in Q$. 
We denote by $b'$ the bound on 
$\card{\adom^\M(s)}$, for any executable situation $s$ of $\D$, 
and on $\card{\adom(\I(q))}$, for $q\in Q$. 
Notice that, in general, $b'$ is different than $b$, in that 
the former bounds the number of \emph{objects} occurring in the interpretations,
while the latter bounds the number of \emph{tuples} in the interpretation of fluents.
To obtain $b'$, observe that if the theory is bounded by $b$, then, for any model,   
the extension of each fluent $F\in\F$ at any executable situation contains at most 
$b$ distinct tuples. Thus, the extension of the generic fluent $F$ cannot contain,
at any executable situation, more than $a_F\cdot b$ distinct objects, where 
$a_F$ is the arity of $F$
(the maximum number of tuples, each with distinct objects, distinct also from all others
in the extension). As a result, the extensions cannot contain, overall, more than 
$\sum_{F\in\F} a_F\cdot b$ distinct objects.
Hence, considering that $\I(q)$ interprets both the fluents in $\F$ and the constants in $C$,
it follows that $\card{\adom(\I(q))}\leq\sum_{F\in\F} a_F\cdot b+\card{C}\doteq b'$.
\end{proof}
Then, we use the obtained bound to show that also 
the set of \emph{all} objects occurring in the labelings 
of some state in $Q$, denoted $\adom(Q)$,
is bounded.
\begin{lemma}\label{lem:AdomQbound}
Let $\adom(Q)=\bigcup_{q\in Q}\adom(\I(q))$.
At any iteration of Procedure~\ref{alg:tf},
we have that  $\card{\adom(Q)}\leq 2b'+N$, for 
$b'$ the bound on $\card{\adom(\I(q))}$ defined as in Lemma~\ref{lem:qbound},
and $N$ the maximum number of parameters of the action types in $\D$.
\end{lemma}
\begin{proof}
By induction on the size of $Q$. 
For $Q=\set{q_0}$,  we have that 
$\adom(Q)=\adom(\I(q_0))$, thus the thesis 
follows as, by Lemma~\ref{lem:qbound}, 
$\card{\adom(\I(q_0))}<b'$.
For $Q=\set{q_0,\ldots,q_n}$, assume, by induction hypothesis, that
$\card{\adom(Q)}< 2b'+N$. 
Since, by Lemma~\ref{lem:qbound}, the state $q\in Q_{te}\subseteq Q$ 
picked at line~\ref{alg:pick}
is such that $\card{\adom(\I(q))}\leq b'$ and $\Delta$ is infinite, then, 
by Theorem~\ref{th:fin-ans} 
(after applying Theorem~\ref{prop:sitcalc-fo-act}, if action terms have to be suppressed in $\phi_F$), 
$\I'$ (line~\ref{alg:ssa}) is such that 
$\adom(\I')\subseteq \adom(\I(q))\cup v(\vec x)$.\footnote{To simplify the 
notation, we use $v(\vec x)$ for the set
$\set{v(x_1),\ldots,v(x_n)}$.}
Now, observe that $\vfo(\vec x)$ may take values from $O$ and
that the constraints on the choice of $O$ 
(line~\ref{alg:choose-O}) 
require that the reuse of objects from $\adom(Q)$ be maximized. 
That is, including fresh objects (with respect to $\adom(Q)$) 
in $O$ is allowed (in fact, required)
\emph{only} if needed to guarantee that 
$\card{O}=\card{\vec x}$ (while $O\cap\adom(\I(q))=\emptyset$).
Thus, two cases are possible: either
$\card{\adom(Q)\setminus \adom(\I(q))}<\card{\vec x}$
(in which case fresh objects must be added to $O$), or not.
In the first case, 
because $\card{\vec x}\leq N$ and $\adom(\I(q))\subseteq \adom(Q)$,
it follows that $\card{\adom(Q)}-\card{\adom(\I(q))}<N$.
Thus, since 
$\card{\adom(\I(q))}\leq b'$, we have that 
$\card{\adom(Q)}<N+b'$. From this, 
observing that $\card{\adom(\I(q'))}\leq b'$, we obtain
$\card{\adom(Q\cup\set{q'})}\leq 2b'+N$.
In the second case, $O$ contains no fresh objects, thus 
$\card{\adom(Q\cup\set{q'})}=\card{\adom(Q)}\leq 2b'+N$.
\end{proof}
Exploiting this result, we can prove termination.
\begin{theorem}\label{thm:termination}
Procedure~\ref{alg:tf} terminates and returns a finite-state transition system $T_F$.
\end{theorem}
\begin{proof}
Firstly, observe that, as a consequence of Lemma~\ref{lem:AdomQbound}:
\myi checking whether $\I(q),v\models \Poss(A(\vec x))$ 
(line~\ref{alg:poss}) is decidable, and \myii $F^{\I'}$ (line~\ref{alg:ssa}) is computable. 
These, indeed, are implied by the fact that $\card{\adom(\I(q))}$ is bounded, thus
finite, and by Theorems~\ref{th:fin-sat} and~\ref{th:fin-comp},
respectively. To apply these theorems, however,
one needs to suppress action terms first, if present, in 
formulas $\phi_F(A(\vec x),\vec y)$ and 
$\phi_A(\vec x)$.
To this end, Theorem~\ref{prop:sitcalc-fo-act} can be used.
Notice also that computability of $F^{\I(q_0)}$ (line~\ref{alg:init3})
is a direct consequence of the fact that 
$\D$ has complete information and is bounded, therefore the extension of all 
fluents at $S_0$ is finite. Items \myi and \myii above guarantee that all the atomic steps of 
Procedure~\ref{alg:tf} can be completed in finite time.

Next, we prove that eventually $Q_{te}=\emptyset$.
Observe that, since
$\A$ (i.e., the set of action types of $\D$), 
$Q$, $O$, $\adom^\M(S_0)$, and $\adom(\I(q))$ are 
finite, it follows that, at every iteration of the while-loop 
(lines~\ref{alg:while}--\ref{alg:end-while}),
the nested loops (lines~\ref{alg:forall-types}--\ref{alg:end-forall-types}) terminate;
thus, proving that $Q_{te}$ becomes empty in a finite number of steps
is sufficient to prove that only a finite number of iterations are executed and, hence,
the procedure terminates. Obviously, this also implies that  the returned $Q$, 
thus $T_F$, is finite.

To see that eventually $Q_{te}=\emptyset$, 
notice that $Q$ is \emph{inflationary}, i.e., states, once added,
are never removed. Consequently, objects can be 
added to $\adom(Q)$ (when a fresh $q'$ is added) but not removed. 
This, together with the fact that, by Lemma~\ref{lem:AdomQbound}, 
$\card{\adom(Q)}$ is bounded, 
implies that, from some iteration $i$ on,
$\adom(Q)$ remains unchanged. 
Let $AQ_i$ be $\adom(Q)$ at iteration $i$ (and at subsequent steps).
Obviously, after that point, if a fresh state $q'$ is added, it must be such that 
$\adom(\I(q'))\subseteq AQ_i$.
Notice that, even though $\adom(Q)$ cannot change, this 
is not the case for $Q$. Indeed, new states $q'$ could still be added, 
as long as $\I(q')=\I'$ contains only objects from $AQ_i$.
However, since $\card{\adom(Q)}$, thus $\card{AQ_i}$, is bounded,
only finitely many  interpretations $\I'$ can be built using values from $AQ_i$.
Consequently, if new states keep being introduced after $i$, it follows that, from 
some step $i'$ on, the interpretation $\I'$ generated at 
line~\ref{alg:ssa} matches the interpretation $\I(q')$ of some $q'$ already 
in $Q$.
Hence, from $i'$ on, the condition at line~\ref{alg:if-add} is always satisfied
(with $h$ the identity function), 
and no fresh state $q'$ can be added to $Q$ any more.
Therefore, no new state is added to $Q_{te}$ (line~\ref{alg:add}), which
becomes eventually empty, as at every iteration one state is extracted from it
(line~\ref{alg:pick}). This completes the proof.
\end{proof}

Finally, we show that the returned $T_F$ retains all the information 
needed to check whether $\M\models\Phi$. That is, 
by Theorem~\ref{thm:mymu-bisimulation}, we show that  $T_F$ is 
persistence-preserving bisimilar to $T_\M$.
\begin{theorem}\label{thm:bisimulation-finite-induced}
The TS $T_F$ computed by Procedure~\ref{alg:tf}, on a basic action theory $\D$ 
(with complete information) bounded by $b$ and a model $\M$ for $\D$, is 
persistence-preserving bisimilar to the TS $T_\M$ induced by $\M$.
\end{theorem}
\begin{proof}
Let $T_F=\tup{\Delta,Q,q_0,\I_F,\ra_F}$ and 
$T_\M=\tup{\Delta,R,r_0,\I_\M,\ra_\M}$,
and define the relation $B\subseteq Q\times H\times R$ 
such that $\tup{q,h,r}\in B$ if and only if 
$\arestrict{\I}_F(q)\sim_h \arestrict{\I}_\M(r)$ (for any $h$).
Notice that, since $T_F$ and $T_\M$ have the same object domain $\Delta$,
$h$ can always be extended to a standard isomorphism $\hat h$ 
between $\I_F(q)$ and $\I_\M(r)$: namely, one can take any bijection 
$\hat h:\Delta\mapsto \Delta$ such that $\restrict{\hat h}{\adom(\I_F(q))}=h$.

We show that $B$ is a persistence-preserving bisimulation between $T_F$ and 
$T_\M$. (page~\pageref{def:bisim}).
Consider a tuple $\tup{q,h,r}\in B$.
Requirement~\ref{bisim:req1} of the definition is trivially satisfied by the 
definition of $B$.
As to requirement~\ref{bisim:req2}, let $q'\in Q$ be such that $q\ra_F q'$.
As shown in the proof of 
Theorem~\ref{thm:termination}, there exists an executable 
situation $s$ such that $\I_F(q)=\I_\M(s)$. 
Moreover, by the definition of $T_\M$, 
$r$ is a situation such that $\I_\M(r)$ matches the interpretation given by 
$\M$ to fluents at $r$.
Because $q\ra_F q'$, by the construction of $T_F$ in 
Procedure~\ref{alg:tf} (line~\ref{alg:poss}), 
we have that, for some valuation $\vfo$ and action type $A$,
$\I_F(q),\vfo\models \Poss(A(\vec x))$,
that is, by the existence of $s$ as above, 
$\M,\vfo\models\Poss(A(\vec x),s)$. 
Then, by extending $h$ to an isomorphism $\hat h$ between 
$\I_F(q)$ and $\I_\M(r)$, as discussed above,
we can see that $\I_\M(r),\vfo'\models\Poss(A(\vec x),r)$, for $\vfo'=\hat h\circ\vfo$, 
which implies that $\M,\vfo'\models\Poss(A(\vec x),r)$.
Therefore, by the definition of $T_\M$, 
for $r'=do^{\M}(A^\M(\hat h(\vfo(\vec x)),r)\in R$, we have that $r\ra_\M r'$.
Thus requirement~\ref{bisim:req2a} is fulfilled.

Next, we show the existence of an isomorphism $\hat h'$
between $\I_F(q')$ and $\I_\M(r')$ that extends
$h$. 
Once proven, this implies the existence of a bijection 
$h':\adom(q)\cup\adom(q')\mapsto\adom(r)\cup\adom(r')$ 
such that 
$\restrict{h'}{\adom(\I_F(q))}=h$
and 
$\arestrict{\I}_F(q')\sim_{\restrict{h'}{\adom(\I_F(q'))}}\arestrict{\I}_\M(r')$.
Indeed, it is sufficient to take $h' = \restrict{\hat h'}{\adom(q)\cup\adom(q')}$.
Thus, the existence of $\hat h'$ implies requirement~\ref{bisim:req2b}.

To prove that such an $\hat h'$ exists, we distinguish two cases:
\myi when the transition $q\ra_F q'$ is added 
at line~\ref{alg:add} (i.e., $q'$ is a fresh state), and \myii when it is 
added at line~\ref{alg:if-add-trans} (i.e., $q'$ is already in $Q$).
For case \myi, observe that $\I_\M(r')$ can be obtained
by applying the right-hand side of the successor-state axiom
of each fluent $F$ to $\I_\M(r)$ (see Theorem~\ref{th:ssa-fo}), 
which is also the way to obtain $\I_F(q')$ from $\I_F(q)$, according to 
Procedure~\ref{alg:tf}. 
Then, since $\hat h$ is an isomorphism 
between $\I_F(q)$ and $\I_\M(r)$, we have that $\I_\M(r)=\hat h(\I_F(q))$,
where $\hat h(\I_F(q))$ denotes the interpretation obtained from $\I_F(q)$ 
by renaming its  objects according to $\hat h$.
Because $\vfo'=\hat h\circ \vfo$, it can be checked that 
$\I_\M(r')=\hat h(\I_F(q'))$, thus $\hat h'=\hat h$ is an isomorphism between
$\I_F(q')$ and $\I_\M(r')$, which obviously extends $h$.
For case \myii, let $\I'$ be the interpretation obtained by applying 
the successor-state axioms to $\I_F(q)$. By the discussion above, we have
that $\I_\M(r')=\hat h(\I')$, while, in general, $\I'\neq \I_F(q')$. However, 
the condition at line~\ref{alg:if-add} guarantees the existence of an isomorphism
$g$ such that $\I'=g(\I_F(q'))$, that is the identity on $\adom(\I_F(q))$.
Now, consider $\hat h'=\hat h\circ g$. Being a composition of isomorphisms,
$\hat h'$ is an isomorphism itself, in particular such that $\I_\M(r')=\hat h'(\I_F(q'))$.
Moreover, $\hat h'$ extends $\restrict{h}{\adom(\I_F(q))}$. This is a 
straightforward consequence of the
facts that $\hat h$ extends $h$ and $g$ is the identity on $\adom(\I_F(q))$, 
which imply that $\hat h'$ matches $h$ on $\adom(\I_F(q))$. 
Thus, requirement~\ref{bisim:req2} is fulfilled.
The proof for requirement~\ref{bisim:req3} follows the same 
argument, with $h$ replaced by its inverse $h^{-1}$.

Since $B$ is a persistence-preserving bisimulation,
the fact that $\tup{q_0,h_0,r_0}\in B$, for $h_0$ the identity, completes the proof.
\end{proof}

Next we prove that checking whether 
$T_F$ satisfies a $\mymu$ formula, is decidable.
\begin{theorem}\label{th:mymu-dec}
	Given a transition system $T=\tup{\Delta,Q,q_0,\I,\ra}$,
	if $Q$ is finite and, for every $q\in Q$, $\adom(\I(q))$ is 
	finite, then for every $\mymu$ formula $\Phi$, 
	checking whether $T\models\Phi$ is decidable.
\end{theorem}
\begin{proof}
	Firstly, by applying Theorem~\ref{th:mymu-drop-act} followed by 
	Theorem~\ref{thm:dom-ind} to the FO components of $\Phi$, we rewrite
	$\Phi$ as an equivalent $\mymu$ (closed) formula $\Phi'$ where no
	action terms occur and whose FO components are domain-independent.
	Once done so, the theorem is a consequence of the finiteness of $Q$ 
	and $\adom(q)$, for $q\in Q$.
	Under these assumptions, $\MODAT{\Phi'}$ is 
	easily computable by recursive applications of the definition
	of $\MODAT{\cdot}$ (page \pageref{def:modat}). 
	In particular, for the base case of  $\Phi'$ a FO formula $\varphi'$, 
	since $\varphi'$ is action-term-free and domain-independent,
	one can apply Theorem~\ref{th:fin-sat}.
	As to quantified variables (outside the $FO$ components), 
	they can be easily dealt with, by the finiteness of $\adom(q)$. 
	The other cases are straightforward.
\end{proof}

Finally, putting all the above results together, we obtain
Theorem~\ref{th:dec-verif}, by  observing that 
one can compute $T_F$ using Procedure~\ref{alg:tf} and then 
check whether $T_F\models \Phi$ by Theorem~\ref{th:mymu-dec}. Termination and correctness 
of this construction are guaranteed by Theorems \ref{thm:mymu-bisimulation},
\ref{thm:termination},  and~\ref{thm:bisimulation-finite-induced}.\footnote{Notice that 
no assumption is made on the object domain $\Delta$ of $\M$ except for it to be infinite. Hence, these results hold also if we assume standard names for object domains, as done in \cite{DBLP:conf/kr/GiacomoLP12}: in that case the object domain is infinite but numerable and coincides with the set of constants $\N$ (this requires a second-order domain closure axiom).}

\section{Dealing with Incomplete Information}\label{sec:incompleteInfo}

In this section, we address the case of partial information 
on the initial situation, by assuming that
$\D_0$ is a set of axioms 
characterizing a possibly infinite 
set of bounded initial databases.
Also in this case, we focus on theories whose models 
have infinite object domains (as we have infinitely many distinct constants).  

We first prove that whenever two models interpret their respective initial situations
in isomorphic ways, they are persistence-preserving bisimilar. 
We observe that this result holds independently of the cardinalities 
of the object domains of the models.

\begin{theorem}\label{th:card-bisim}
	Let $\D$ be a bounded basic action theory.
	For every two models $\M$ and $\M'$ of $\D$,  with possibly different infinite object domains 
	$\Delta$ and $\Delta'$, respectively, 
	if $\arestrict{\I}_\M(S^\M_0)\sim\arestrict{\I}_{\M'}(S^{\M'}_0)$, 
	then $T_\M\approx T_{\M'}$.
\end{theorem}
\begin{proof}
Let $T_\M=\tup{\Delta,Q,q_0,\ra,\I}$ and 
$T_{\M'}=\tup{\Delta',Q',q'_0,\ra',\I'}$.
We prove a stronger claim, i.e., that the relation $B\subseteq Q\times H\times Q'$ 
such that $\tup{q_1,h,q_2}\in B$ if and only if 
$\arestrict{\I}(q_1)\sim_h\arestrict{\I'}(q_2)$ (for any $h$), 
is a persistence-preserving bisimulation relation
between $T_\M$ and $T_{\M'}$. 
This result, once proven, implies the thesis;
indeed, by $\arestrict{\I}_\M(S^\M_0)\sim\arestrict{\I}_{\M'}(S^{\M'}_0)$, we have that there exists 
$\bar h$ such that $\arestrict{\I}(S^\M_0)\sim_{\bar h}\arestrict{\I}(S^{\M'}_0)$,
thus, by the definition of  $B$, $\tup{S^\M_0,\bar h,S^{\M'}_0}\in B$, that is, 
$\tup{q_0,\bar h,q'_0}\in B$, as $q_0=S^\M_0$ and $q_0'=S^{\M'}_0$.

Let $\tup{q_1,h,q_2}\in B$. 
Requirement~\ref{bisim:req1} of the definition of 
bisimulation (page~\pageref{def:bisim})
is clearly satisfied.
For requirement~\ref{bisim:req2}, 
first recall that, by definition of induced transition system (page~\pageref{def:induced-ts}),
$\I(q_1)=\I_\M(q_1)$
and
$\I'(q_2)=\I_{\M'}(q_2)$, thus 
$\arestrict{\I}_\M(q_1)\sim_h \arestrict{\I}_{\M'}(q_2)$.
Assume that there exists 
$q_1'\in Q$ such that $q_1\ra q'_1$.
By definition of transition system induced by 
$\M$ (page~\pageref{def:induced-ts}),
there exist an action type $A$ and a valuation $v$ 
such that $\M,v\models \phi_A(\vec x,q_1)$, for 
$\Poss(A(\vec x),s)\equiv\phi_A(\vec x,s)$ the 
precondition axiom of $A$. This is equivalent to 
$\I_\M(q_1),v\models \phi_A(\vec x)$, for 
$\phi_A(\vec x)$ the situation-suppressed version of $\phi_A(\vec x,s)$.
Now, let $\phi'_A(\vec x)$ be the domain-independent version of 
$\phi_A(\vec x)$.
By Theorem~\ref{thm:dom-ind}, we have that
$\I_\M(q_1),v\models \phi_A(\vec x)$ if and only if 
$\arestrict{\I}_\M(q_1),v\models \phi'_A(\vec x)$.
If we extend $h$ to $v(\vec x)$ in a way such that we obtain a bijection $\hat h$
(by a cardinality argument, this is always possible), 
then, because $\arestrict{\I}_\M(q_1)\sim_h \arestrict{\I}_{\M'}(q_2)$, 
we have that 
$\arestrict{\I}_\M(q_1),v\models \phi'_A(\vec x)$ if and only if 
$\arestrict{\I}_{\M'}(q_2),\hat h\circ v\models \phi'_A(\vec x)$.
But then, again by Theorem~\ref{thm:dom-ind}, 
$\I_{\M'}(q_2),\hat h\circ v\models \phi_A(\vec x)$.
Thus, by reintroducing the situation argument in $\phi_A$, we have 
that $\M',v'\models\phi_A(\vec x,q_2)$, that is, there exists an action 
$a'=A^{\M'}(\hat h(v(\vec x)))$ such that $\tup{a',q_2}\in\Poss^{\M'}$.
Therefore, by the definition of $T_{\M'}$, it follows that 
$q_2\ra q'_2$, for $q'_2=do^\M(a',q_2)$. This proves requirement~\ref{bisim:req2a}.

For requirement~\ref{bisim:req2b}, we first show that
$\arestrict{\I}_\M(q'_1)$ can be obtained from $\arestrict{\I}_\M(q_1)$,
through the successor-state axioms.
To this end, notice that $\I_\M(q'_1)$ can be obtained by taking,
for each fluent $F$, the right-hand side 
$\phi(\vec x,a,s)$ of the
corresponding successor-state axiom (the subscript $F$ is removed to simplify the notation),
then deriving the equivalent action-term-free formula $\phi(\vec y,\vec x)$,
as shown in Theorem~\ref{th:ssa-fo}, for action $a=A^\M(\vfo(\vec x))$,
and finally letting $F^{\I_\M(q_1')}=\phi^{\I_\M(q_1)}_{\vec x/\vfo(\vec x)}$, that is, 
by interpreting each $F$ as the answer to the corresponding query 
$\phi$ on the interpretation $\I_\M(q_1)$, under the partial assignment 
$\vec x/\vfo(\vec x)$ (constants are always interpreted as in $\M$). 
Now observe that, since the action theory is bounded, so is
the extension of each fluent $F$ at $q_1$ and $q'_1$. 
Thus, by Theorem~\ref{th:fin-ans}, the extension of each fluent at $q'_1$
contains only values from $\adom(\I_\M(q_1))\cup v(\vec x)$, that is 
$\adom(\I_\M(q'_1))\subseteq\adom(\I_\M(q_1))\cup v(\vec x)$.
Hence, if we denote (for each $F$) the domain-independent 
rewriting of $\phi(\vec y,\vec x)$ as $\phi'(\vec y,\vec x)$, 
by Theorem~\ref{thm:dom-ind},
we have that 
$F^{\I_\M(q_1')}=\phi^{\I_\M(q_1)}_{\vec x/\vfo(\vec x)}=
{\phi'}^{\arestrict{\I}_\M(q_1)}_{\vec x/\vfo(\vec x)}$, that is, by answering 
$\phi'$ on $\arestrict{\I}_\M(q_1)$, we obtain the extension of $F$ 
at $q'_1$. Obviously, by doing so for every fluent $F$, we can obtain
$\arestrict{\I}_\M(q'_1)$ from $\arestrict{\I}_\M(q_1)$.
By an analogous argument, it can be shown that
$\arestrict{\I}_{\M'}(q'_2)$ can be obtained
from $\arestrict{\I}_{\M'}(q_2)$, for 
action $a'=A^{\M'}(\hat h(v(\vec x)))$.

Next, consider again the bijection $\hat h$ defined above, and
recall that 
$\hat h$ extends $h$ on $v(\vec x)$, 
and that $\arestrict{\I}_\M(q_1)\sim_h \arestrict{\I}_{\M'}(q_2)$.
By the invariance of FO under isomorphic interpretations,
we have that, for each fluent $F$, the answers to $\phi'$ on $\arestrict{\I}_\M(q_1)$
and $\arestrict{\I}_{\M'}(q_2)$, under the partial assignments, respectively,
$\vec x/\vfo(\vec x)$ and $\vec x/\hat h(\vfo(\vec x))$, coincide,
modulo the object renaming induced by $\hat h$. 
But then, it is immediate to check that 
$h'=\restrict{\hat h}{\adom(\I_\M(q_1))\cup \adom(\I_\M(q'_1))}$
is a bijection such that 
$\arestrict{\I}_\M(q'_1)\sim_{\restrict{h'}{\adom(\I_\M(q'_1))}} \arestrict{\I}_{\M'}(q'_2)$
and, hence, by the definition of $B$, $\tup{q'_1,\restrict{h'}{\adom(\I_\M(q'_1))},q'_2}\in B$.
This proves requirement~\ref{bisim:req2b}.
The proof of requirement~\ref{bisim:req3b} is analogous.
\end{proof}

Now, consider a set $\Mod$ of models of $\D$ having isomorphic 
interpretations at $S_0$. 
By Theorem~\ref{th:card-bisim}, all such models have induced TSs 
that are persistence-preserving bisimilar to each other. 
Thus, by Theorem~\ref{thm:mymu-bisimulation}, to check whether a 
$\mymu$ formula $\phi$ holds in all models of $\Mod$, 
one can perform the check on any arbitrary model of $\Mod$, using, e.g.,
the technique discussed for the case of complete information. 
This result, together with the assumption of boundedness, 
will be exploited next, to prove our main theorem. 

\begin{theorem}\label{th:inc-inf}
  Let $\D$ be an action theory bounded by $b$ with incomplete
  information on the initial situation, and let $\Phi$ be a $\mymu$
  closed formula.
	Then, checking whether $\D\models\Phi$ is decidable.
\end{theorem}

\begin{proof}
Let $\Mod_\D$ be the set of all models of $\D$, and consider a partition of it
such that each cell contains only models whose interpretations
at $S_0$ match, modulo object renaming. 
Formally, we define 
$\Mod_\D=(\Mod^1_\D,\Mod^2_\D,\ldots)$ such that, for every two models 
$\M$ and $\M'$ in $\Mod_\D^i$, $\arestrict{\I}_\M(S^\M_0)\sim\arestrict{\I}_{\M'}(S^{\M'}_0)$.
As a consequence of the boundedness of $\D$, the number of 
cells in the partition is finite. 
Indeed,  
a bounded number of objects yields, up to object renaming, only a bounded 
number of possible interpretations (of finitely many fluents and constants) at $S_0$. 
Thus, for some finite $n$ depending on the theory $\D$ and the bound $b$, we have that 
$\Mod_\D=(\Mod^1_\D,\Mod^2_\D,\ldots,\Mod^n_\D)$.

Since, by Theorem~\ref{th:card-bisim}, any two models $\M$ and $\M'$ of the generic cell 
$\Mod^i_\D$ induce persistence-preserving bisimilar transition systems, then, by 
Theorem~\ref{thm:mymu-bisimulation}, we have that all the models of 
$\Mod^i_\D$ satisfy $\Phi$ if and only if some model $\M$ of $\Mod^i_\D$
satisfies $\Phi$. Thus, to check whether $\D\models\Phi$, we can simply choose 
one model $\M_i$ per cell $\Mod^i_\D$, and then check whether,
for all $i=1,\ldots,n$, $\M_i\models \Phi$; if this is the case, 
then, and only then, we can conclude that $\D\models\Phi$.
Obviously, for this approach to be effective,
we need a model $\M_i$ per cell $\Mod_\D^i$ and a way to perform the check. 
The rest of the proof addresses these two points.

Let $\F$ be the set of situation-suppressed 
fluents of $\D$, and $C$ the (finite) set 
of constant symbols explicitly mentioned in $\D$ (beyond $\D_{uno}$).
We observe that each cell $\Mod_\D^i$ of the partition 
$\Mod_\D=(\Mod_\D^1,\ldots,\Mod_\D^n)$ 
can be uniquely identified by an interpretation $\I_i$ of $\F$ and $C$ over some 
infinite object domain $\Delta$. Indeed, by transitivity of $\sim$,
any two models $\M,\M'$ of $\D$ such that $\arestrict{\I}_\M(S_0^\M)\sim\I_i$ 
and $\arestrict{\I}_\M'(S_0^{\M'})\sim\I_i$ are also such that
$\arestrict{\I}_\M(S_0^\M)\sim\arestrict{\I}_\M'(S_0^{\M'})$.
Notice that $\I_i$ certainly exists, as one can simply take 
$\arestrict{\I}_\M(S_0^{\M})$, for some model $\M\in\Mod_\D^i$.
Clearly, each $\I_i$ contains only a bounded number of objects in the active domain
and satisfies $\D_0$, i.e., $\I_i\models \D_0$.

Now, assume given one interpretation 
$\I_i$ per cell $\Mod_\D^i$ (we show below how to obtain them)
and observe that, 
from $\I_i$, we can extract a complete initial situation description
as a database $\D_0^i$. This can be easily done, as $\I_i$ is finite.
Consider the theory $\D^i=(\D\setminus\D_0)\cup \D_0^i$, obtained by replacing 
$\D_0$ with $\D_0^i$, and assume
the same interpretation of constants in $C$ as that defined by $\I_i$.
Under this assumption, $\D^i$ defines a family of models 
that differ only in the object domain and in the interpretation of constants outside 
$C$ (which, however, must satisfy $\D_{uno}$). 
In particular, 
the interpretation of fluents in $\F$ and constants in $C$, at $S_0$,
of all such models, is the same as that of $\I_i$.
Thus, the models of $\D^i$ constitute a subset of $\Mod_\D^i$. 
To isolate one of such models, we fix an arbitrary 
infinite object domain $\Delta$ (such that $\adom(\I_i)\subseteq\Delta$),
and arbitrarily extend the partial interpretation of constants 
over the constants outside $C$, satisfying $\D_{uno}$.
Notice that this can always be done, as $\Delta$ is infinite and the set 
of constant symbols countable.
  With $\Delta$ and the denotation of all constants
  fixed, $\D^i$ has complete information, i.e., yields a 
single model $\M_i$, thus,
by Theorem~\ref{th:dec-verif}, we can check whether $\D^i\models \Phi$
, i.e., whether 
$\M^i\models \Phi$ (notice that, as it turns out from Procedure~\ref{alg:tf}, 
to perform the check, one does not even need 
to know the interpretation of constants outside $C$).  
This, by the discussion above, is equivalent to 
checking whether for all models $\M\in\Mod_\D^i$, it is the case that $\M\models \Phi$.
Therefore, if the set of interpretations $\Gamma=\{\I_1,\ldots,\I_n\}$ is given, we can check whether 
$\D\models \Phi$.

It remains to explain how such a set of interpretations $\Gamma=\{\I_1,\ldots,\I_n\}$ can be obtained.
To this end, observe that, by Lemma~\ref{lem:qbound}, 
it follows that $\card{\adom(\I_i)}\leq\sum_{F\in\F} a_F\cdot b+\card{C}\doteq b'$.
Based on this, the set $\Upsilon$ of interpretations $\I_i$ can be obtained by:
\myi fixing a set $O$ of $b'$ arbitrary objects; \myii generating a set $\Upsilon'$ of 
all the finitely many interpretations of $\F$ and $C$ over $O$,
such that $\D_{uno}$ is enforced on $C$ and for every interpretation $\I'\in\Upsilon'$, 
$\I'\models\D_0$;
\myiii for any set $\Upsilon''\subseteq\Upsilon'$ of isomorphic interpretations, 
removing from $\Upsilon'$ all but one of such interpretations (in fact, this step is not 
needed to our purposes,  but avoids useless redundancies). 
The resulting $\Upsilon'$ is the set of desired interpretations $\I_1,\ldots,\I_n$, which we rename
simply as $\Upsilon$.

Now, observe that, by the way it is defined, $\Upsilon$ contains, 
up to object renaming, all possible interpretations of $\F$ and $C$ 
over a set of $b'$ distinct objects, that satisfy $\D_0$ and $\D_{uno}$ (on $C$).
Thus, since for a generic model $\M$ of $\D$, the interpretation 
$\arestrict{\I}_\M(S_0^\M)$ contains at most 
$b'$ distinct objects (by the boundedness of $\D$), it turns out that 
there exists an interpretation $\I_i\in\Upsilon$ such that 
$\arestrict{\I}_\M(S_0^\M)\sim \I_i$. Therefore, the cell $\Mod_\D^i$
such that $\M\in\Mod_\D^i$,
is characterized by some interpretation $\I_i\in\Upsilon$, namely the 
interpretation at $S_0$ shared, up to object renaming, by the models
of the cell itself.
On the other hand, because any $\I_i\in\Upsilon$ enforces $\D_{uno}$ 
and is such that $\I\models\D_0$, it follows that there exists some 
model $\M$ of $\D$ such that $\arestrict{\I}_\M(S_0^\M)\sim \I_i$. Therefore,
every interpretation of $\Upsilon$ characterizes some cell $\Mod_\D^i$,
specifically, that of the models $\M$ such that $\arestrict{\I}_\M(S_0^\M)\sim \I_i$.
Therefore, $\Upsilon$ is indeed the set of desired intepretations.
This concludes the proof.
\end{proof}

This result, besides stating decidability of the verification problem under 
incomplete information, provides us with an actual procedure 
to perform verification in this case. 


\section{Computational Complexity}\label{sec:complexity}

In this section, we asses the computational complexity of verifying $\mymu$ formulas over a bounded situation calculus basic action theory $\D$. In particular we show that the constructive techniques we have used for proving decidability are, in fact, optimal with respect to worst case computational complexity. 
We make the assumption that, for a basic action
theory $\D$, the maximum number of distinct objects occurring in the
state of any situation, dominates the input size of $\D$
itself, 
and that there exists a bound $\bar a_F$ on the maximum arity of
fluents.
This is a reasonable assumption, analogous to that, typical in databases, that the size 
of the database provides a higher bound on the size of the input along all 
dimensions, and that,  in practical cases, there exists an upper bounds on the 
arity of relations.
We exploit the constructive techniques introduced for showing decidability to get an exponential time upper-bound.

\begin{theorem}
  Verifying $\mymu$ formulas over a situation calculus basic action
  theory bounded by $b$, with complete information on the initial
  situation, can be done in time exponential in $b$.
\end{theorem}
\begin{proof} This is a consequence of Procedure~\ref{alg:tf} and the complexity of 
\mymu model checking.
Firstly, consider Procedure~\ref{alg:tf} and observe that, 
by Lemma~\ref{lem:AdomQbound}, 
at any iteration, the number $m$ of distinct objects occurring, overall, 
in the interpretations of states (i.e.~$\card{\adom(Q)}$ of Lemma~\ref{lem:AdomQbound}) 
is bounded by $2b'+N$, where 
$b'=\sum_{F\in\F}b\cdot a_F$, $a_F$ is the arity of fluent $F$, 
and $N$ is the maximum number of 
parameters in action types. 
Since we assume $\card{F}$ and $N$ bounded by $b$, and $a_F$ 
bounded by a constant, it turns out that $m$ is polynomial in $b$.
Now, observe that, with $m$ distinct objects and $a_F$ bounded by a 
constant, one can obtain a 
number of interpretations of $\F$ and $C$ that is at most  exponential in $m$, i.e., in 
(a polynomial of) $b$. 
Then, because in Procedure~\ref{alg:tf} every state is associated with exactly one interpretation, and since no state is visited more than once, we have that 
the while-loop (lines~\ref{alg:while}--\ref{alg:end-while}) terminates after, at most, an
exponential number of iterations.

As to each iteration, by our assumptions, we have that any loop inside 
the while-loop ends after  at most exponentially many iterations. 
Indeed, for any action type with at most $N$ parameters, we have
at most $m^N$ possible assignments, thus $m^N\leq m^{b'}$, which gives 
an exponential bound, as both $m$ and $b'$ are polynomial
with respect to $b$. 
Now, observe that the dominant operation in the while-loop 
is checking whether two interpretations are isomorphic. 
Since also this check
can be performed in exponential time with respect to $b$ (the problem is in NP),
we obtain, overall, an exponential time-bound for Procedure~\ref{alg:tf}.

Now, recall that propositional $\mu$-calculus model checking is polynomial 
with respect to the sizes of the input transition system and the input formula~\cite{Emerson96}.
As to the transition system, the check is performed on the one returned by 
Procedure~\ref{alg:tf}, which has size at most exponential in $b$ (i.e., as
many interpretations as one can obtain with at most $m$ objects, plus a quadratic number
of transitions wrt it).
As to the formula, say $\Phi$, we first rewrite it (in polynomial time) into its 
equivalent domain-independent version $\Phi'$, and then ``propositionalize''
it, by quantifier elimination, using only the values that occur, overall, in the 
active domains of the interpretations of the states of the input transition system. 
This step can be done,
again, in exponential time, and returns a quantifier-free formula exponentially larger than the original one, but equivalent to it, on the obtained finite transition system. 
Thus, since $\mu$-calculus model checking is polynomial wrt the size of both the 
transition system and the formula, we obtain that, overall, the check requires time 
at most exponential wrt $b$.
\end{proof}


Such an exponential bound is, in fact, tight, as we can show the EXPTIME-hardness of the problem by reduction from acceptance in a polynomial-space bounded alternating Turing machine. 
\begin{theorem}
Verifying $\mymu$ formulas over bounded situation calculus basic action theories with complete information on the initial situation is EXPTIME-hard.
\end{theorem}
\begin{proof}
We show a reduction from polynomial-space bounded alternating Turing machines, whose acceptance problem is EXPTIME-complete \cite{ChKS81}.
An (one-tape) \emph{Alternating Turing Machine} (ATM) \cite{ChKS81} is a tuple 
$M=(Q,\Gamma,\delta,q_0,g)$ where
\begin{itemize}
\item $Q$ is the finite set of states;
\item $\Gamma$ is the finite tape alphabet;
\item $\delta: Q\times\Gamma\times Q\times\Gamma\times\{L,R\}$ is called the transition table ($L$ shifts the head left and $R$ shifts the head right); 
\item $q_0\in Q$ is the initial state;
\item $g:Q\rightarrow\{and,or,accept\}$ specifies the type of each state.
\end{itemize}
If $M$ is in a state $q\in Q$ with $g(q)=accept$ then that
configuration is said to be \emph{accepting}. A configuration with
$g(q)=and$ is said to be accepting if \emph{all} configurations
reachable in one step are accepting.  A configuration with $g(q)=or$
is said to be accepting when there \emph{exists} some configuration
reachable in one step which is accepting.  (The latter is the type of all
states in a Nondeterministic Turing Machine.) $M$ is said to accept
an input string $w$ if the initial configuration of $M$ (where the state of
$M$ is $q_0$, the head is at the left end of the tape, and the tape
contains $w$) is accepting. 
An ATM is said to be \emph{polynomial-space-bounded} if it scans at most a number of tape cells that is polynomially-bounded by the size of the input. 


Following \cite{Reiter01-Book} (Chap.\ 4), we can axiomatize the ATM using the following fluents:

\begin{itemize}
\item $transTable(q, c, q', c', m, s)$. This is a
  situation-independent predicate (i.e., with a trivial
  successor-state-axioms preserving its content forever) describing
  the ATM's transition table $\delta$:
when in state $q$ scanning tape symbol
  $c$, the machine enters state $q'$, overwrites $c$ with tape symbol
  $c'$, and moves its tape head in the direction $m$, which is one of
  $L$ (left) or $R$ (right).

\item $gType(q,t,s)$. This is a situation-independent predicate
  assigning (once and for all) a type $t\in \{ and,or, accept \}$
to the state $q$ of the ATM. 

\item $cell(i,c,s)$. This means that tape cell $i \in [0,\ldots,\ell]$ contains the
  symbol $c \in\Gamma\cup \{blank\}$ in situation $s$. Notice that in
  every situation the number of facts of the form
  $cell(i,\gamma,s)$ is fixed and determined by the maximal length of the
  tape of the bounded ATM, $\ell$.  Initially, the first cells contains the
  input word $w$ while the others are $blank$.

\item $state(q,s)$. This means that in situation $s$, the machine's state is
  $q$. Initially, we have $state(q_0,S_0)$, where $q_0$ is the initial
  state of the ATM.

\item $scan(i, s)$. This means that the machine's head is scanning
  tape cell $i \in [0,\ldots,\ell]$ in
  situation $s$. Initially, the head is scanning tape cell $0$. In any
  situation, there will only be one fact of the form $scan(i, s)$.
\end{itemize}
We need just one action type $trans(q',c', m)$, meaning that the
machine makes a transition from the current configuration to a new
configuration where the state is $q'$, tape symbol
$c'$ is written, and the tape head moves in direction $m$,
whose precondition axiom is as follows:
\[\begin{array}{rcl}
Poss(trans(q',c', m),s) &\equiv& \exists q, i, c.\,
state(q,s) \land scan(i,s) \land cell(i,c,s) \land {}\\
&&\qquad\qquad\qquad transTable(q, c, q', c', m, s)
  \end{array}
  \]
The successor state axioms for the fluents that can change are as
  follows:
\[\begin{array}{l}
state(q,do(a,s) \equiv \exists c,m. a = trans(q,c,m) \lor {}\\
\hspace{8.2em} state(q,s) \land \neg \exists
    q',c,m. a = trans(q',c,m) \land q' \neq q
\end{array}\]
\[\begin{array}{l}
scan(i,do(a,s) \equiv\\
\qquad \exists q,c,i'. a = trans(q,c,L) \land scan(i',s) \land {}\\
\hspace{5em} (i' = 0 \supset i = i') \land (i' \neq 0 \supset i = i' - 1) \lor {}\\
\qquad \exists q,c. a = trans(q,c,R) \land scan(i',s) \land i = i' + 1
    \lor {}\\
\qquad scan(i,s) \land \neg \exists q,c,m. a = trans(q,c,m)
\end{array}\]
\[\begin{array}{l}
cell(i,c,do(a,s) \equiv
\exists q,m. a = trans(q,c,m) \land scan(i,s) \lor {}\\
\qquad \qquad cell(i,c,s) \land \neg \exists q,c',m. a = trans(q,c',m) \land
    scan(i,s) \land c' \neq c
\end{array}\]
For initial situation description, assuming the input $w = c_0 \ldots
c_i$, we have:
\[\begin{array}{l}
state(q_0,S_0), scan(0,S_0),\\
cell(0,c_0,S_0), \ldots, cell(i,c_i,S_0),\\
cell(j,blank,S_0),\ \mbox{for}\ j \in [i,\ldots, \ell]
\end{array}\]

Acceptance of the ATM is defined using the following $\mymu$ formula $\Phi$:
\[\begin{array}{l}
\mu Z.\, (\exists q. state(q) \land gType(q,accept) \lor {}\\
\phantom{\mu Z.\,} (\exists q. state(q) \land gType(q,and)) \land
    \BOX{Z} \lor {}\\
\phantom{\mu Z.\,} (\exists q. state(q) \land gType(q,or)) \land \DIAM{Z}
  \end{array}
  \]
Then we have that $\D \models \Phi$ if and only if $M$ accepts $w$.
Notice that in any situation there is exactly one fact of the form $gType(q,t,s)$. 
Notice also that the above condition does not require quantification across situations. 
\end{proof}


\section{Checking Boundedness}\label{sec:boundedness}

We now show that we can always check whether any BAT maintains
boundedness for a given bound. That is, if the initial situation
description is bounded, then the entire theory is too (for all
executable situations).

First notice that we can determine in a situation $s$ whether
every executable action $a$ if performed next does not exceed the bound
(i.e.\ in $do(a,s)$).
We can capture the notion of a fluent $F$ being bounded at the next step by the formula:

\[\bigwedge_{A\in\A} \forall \vec{x}. Poss(A(\vec{x}),s) \limp Bounded_{F,b}(do(A(\vec{x}),s)).\]

\noindent
Notice that each $Bounded_{F,b}(do(A(\vec{x}),s))$ is regressable through $A(\vec{x})$. 
As a result the formula above is equivalent to a first-order situation calculus formula uniform in $s$; 
we call the latter formula $\NextBounded_{F,b}(s)$, and 
we call $\NextBounded_{b}(s)$ the formula $\bigwedge_{F\in\F}
\NextBounded_{F,b}(s)$.

To check that the theory is bounded  by $b$ it is sufficient to verify that the theory entails the temporal formula:
\[AG\NextBounded_{b} \doteq \nu Z. \NextBounded_{b} \land \BOX{Z},\]
\noindent
which expresses that always along any path $\NextBounded_{b}$ holds.
Unfortunately deciding whether this formula is entailed by the action theory is directly doable with the techniques in previous sections only if the theory is bounded, which is what we want to check. 
However it turns out that we can construct a modified version of the
action theory that is guaranteed to be bounded and that we can use to do the checking.  

Let $\D$ be the action theory. We define a new action theory $\D\D$ obtained by augmenting $\D$ as follows:
\begin{itemize}
\item $\D\D_{S_0} = \D_{S_0} \cup \{\phi[\vec{F}/\vec{F'}] | \phi \in \D_{S_0}\}$

\item 
$\D\D_{SS} = \D_{SS} \cup {} \{ F'(\vec{x},do(a,s)) \equiv
\Phi(\vec{x},a,s) \land \NextBounded_{b}(s) \mid {}$\\
\hspace*{17em}$ F(\vec{x},do(a,s)) \equiv \Phi(\vec{x},a,s) \in \D_{SS}\}
 $

\item $\D\D_{ap} =  \{ Poss(A(\vec{x}),s) \equiv
  \Psi(\vec{x},a,s)\land \NextBounded_{b}(s) \mid {}$\\ 
\hspace*{16em} $Poss(A(\vec{x}),s) \equiv \Psi(\vec{x},a,s)  \in \D_{AP}\}$
\end{itemize}
Intuitively $\D\D$ extends $\D$ with primed copies of fluents, which are axiomatized to act, in any situation, as the original ones as long as the \emph{original theory remains bounded} by $b$ in that situation, otherwise they become empty (and actions cannot be executed according to $Poss$.)
It is easy to show the following key property for $\D\D$.
\begin{lemma}\label{thm-dd-equivalence}
\[\D\D\models \forall s. (\forall \hat{s}. \hat{s} < s \limp \NextBounded_{b}(\hat{s})) \limp \forall \vec{x}. (F'(\vec{x},s) \equiv F(\vec{x},s)).\]
\end{lemma}
\begin{proof}
By induction on situations.
\end{proof}

Now we define a new action theory $\D'$ which can be considered a sort of projection of $\D\D$ over the primed fluents only.
Let $\D'$ be: 
\begin{itemize}
\item $\D'_{S_0} = \{\phi[\vec{F}/\vec{F'}] | \phi \in \D_{S_0}\}$.

\item 
$\D'_{SS} =  \{ F'(\vec{x},do(a,s)) \equiv
         \Phi[\vec{F}/\vec{F'}](\vec{x},a,s) \land \NextBounded_{b}[F/F'](s) \mid {}$\\
\hspace*{17em} $F(\vec{x},do(a,s)) \equiv \Phi(\vec{x},a,s) \in \D_{SS}\}$

\item 
$\D'_{ap} =  \{ Poss(A(\vec{x}),s) \equiv
         \Psi[\vec{F}/\vec{F'}](\vec{x},a,s)\land \NextBounded_{b}[F/F'](s) \mid {}$\\
\hspace*{16em} $Poss(A(\vec{x}),s) \equiv \Psi(\vec{x},a,s)  \in \D_{AP}\}$
\end{itemize}
Notice that $\D'$ is bounded by construction if $\D'_{S_0}$ is, and furthermore it preserves the information about the \emph{original theory} being bounded at the next step, though in terms of primed fluents.
Exploiting the above lemma on $\D\D$ and the construction of $\D'$, 
we can show that $\D'$ has the following notable property:
\begin{lemma}\label{thm-boundeness}
\[\D\models AG\NextBounded_{b}(S_0) ~\mbox{ iff }~
\D'\models AG\NextBounded_{b}[\vec{F}/\vec{F'}](S_0).
\footnote{Notice that $\NextBounded_{b}[\vec{F}/\vec{F'}]$ expresses that in the original theory the next situations are bounded, though now syntactically replacing original fluents with their primed version.}\]
\end{lemma}

\begin{proof}
By Lemma~\ref{thm-dd-equivalence}, it is immediate to see that $\D\models AG\NextBounded_{b}(S_0)$ implies $\D'\models AG\NextBounded_{b}[\vec{F}/\vec{F'}](S_0)$.
For the opposite direction,  suppose that $\D'\models AG\NextBounded_{b}[\vec{F}/\vec{F'}](S_0)$, but $\D\models AG\NextBounded_{b}(S_0)$ does not hold. This means that there exists a model of $\D$ and a situation $S$ where $\lnot \NextBounded_{b}(S)$ holds, though in all previous situations $s < S$  we have that  $\NextBounded_{b}(s)$ holds. Now by Lemma~\ref{thm-dd-equivalence}, we can construct a model for $\D'$ such that the truth values of $F$ are replicated in $F'$ as long as $\NextBounded_{b}$ holds in the previous situation. So in $S$, we must have  $\lnot \NextBounded_{b}[\vec{F}/\vec{F'}](S)$, which contradicts the assumption that $\D'\models AG\NextBounded_{b}[\vec{F}/\vec{F'}](S_0)$.
\end{proof}

By Lemma~\ref{thm-boundeness}, since $\D'$ is bounded by $b$ if
$\D'_{S_0}$ is, it follows that:

\begin{theorem}\label{thm-decidability}
  Given a BAT whose initial situation description is bounded by $b$, then checking
  whether the entire theory is bounded by $b$ is decidable.
\end{theorem}

Notice that we pose no restriction on the initial situation description except that it is representable in first-order logic, hence checking its boundedness remains undecidable: 
\begin{theorem}\label{thm-boudendess-undec}
Given a FO description of the initial situation $\D_0$ 
and a bound $b$, it is undecidable to check whether all models of  $\D_0$ are bounded by $b$. 
\end{theorem}
\begin{proof}
By reduction to FO unsatisfiability.
Suppose we have an algorithm to check whether a FO theory $\D_0$ is bounded by 0. 
Then we would have an algorithm to check (un)-satisfiability of $\D_0$. 
Indeed consider for a fixed fluent $\hat{F}$:
\[\hat{\D}_0 = (\D_0\land \exists \vec{x}. \hat{F}(\vec{x}, S_0)) \lor 
(\bigwedge_{F\in\F}  \forall \vec{x}. \lnot F(\vec{x}, S_0)) \]
\noindent
Note that $\bigwedge_{F\in\F} \forall \vec{x}. \lnot F(\vec{x}, S_0)$ has only models bounded by 0, while $\exists \vec{x}. \hat{F}(\vec{x}, S_0) $ has only models with at least one tuple (and thus one object) in $\hat{F}$. Hence we get that
$\hat{\D}_0$ is bounded by 0 iff $\D_0$ is unsatisfiable.  A similar argument holds for every bound $b$.
\end{proof}
Nonetheless in many cases we know  by construction that the initial situation is bounded. 
In such cases the proof technique of Theorem~\ref{thm-decidability} provides an effective way to 
check if the entire theory is bounded.

\section{Related Work}\label{sec:rel-work}











Besides the situation calculus \cite{McCarthy1969:AI,Reiter01-Book},
many other formalisms for reasoning about actions have been developed in
AI, including the event calculus
\cite{DBLP:journals/ngc/KowalskiS86,DBLP:books/daglib/0095085,DBLP:books/sp/wooldridgeV99/Shanahan99},
the features and fluents framework \cite{SandewallBook94}, action
languages such as $\A$
\cite{DBLP:journals/jlp/GelfondL93} and $\C+$ \cite{DBLP:journals/ai/GiunchigliaLLMT04},
the fluent calculus \cite{DBLP:journals/ai/Thielscher99}, and many others.  In
most of these, the focus is on addressing problems in the
representation of action and change, such as the frame problem.
 Some attention has also been paid to specifying and verifying general
 temporal properties, especially in the context of planning.  
The Planning Domain Definition Language
 (\propername{PDDL}) \cite{PDDL98} has been developed for specifying
 planning domains and problems,
 and a recent version supports the expression of temporal constraints
 on the plan trajectory \cite{GereviniLong06}.  Approaches such as
 those in TLPlan \cite{DBLP:journals/amai/BacchusK98}, in TALplanner
 \cite{DBLP:journals/amai/KvarnstromD00}, or in planning via model
 checking \cite{PistoreTraverso:IJCAI01} support planning with such
 temporal constraints. Within the situation calculus, temporal
 constraints for planning have been studied in, e.g.,
 \cite{DBLP:conf/kr/BienvenuFM06,DBLP:conf/aips/BaierM06}.  All these
 planning-related approaches are essentially propositional and give
 rise to transition systems that are finite-state. One interesting
 attempt to interpret first-order linear temporal logic simultaneously
 as a declarative specification language and procedural execution
 language is that of \propername{MetateM} \cite{BFGGO95}, though
 verification is not addressed.

Most work on verification has been done in computer science,  generally focusing on
finite-state systems and programs.
Many logics have been developed to specify temporal properties of
such systems and programs, including linear-time logics, such as
Linear Temporal Logic (LTL) \cite{PnuelliLTL97} and
Property-Specification Language (PSL) \cite{EiFi06}, and branching
time logics such as Computation Tree Logic (CTL) \cite{ClarkeE81CTL}
and CTL$^*$~\cite{EmHal83CTL*}, the
$\mu$-calculus~\cite{Emerson96,BS07}, which subsumes the previous two,
as well as Propositional Dynamic Logic (PDL) \cite{FisherLadnerPDL79},
which incorporates programs in the language.  Model checking (and
satisfiability) in these propositional modal logics is decidable
\cite{BaKG08}, but they can only represent finite domains and finite
state systems.  Practical verification systems, e.g.,
\cite{DBLP:journals/tse/Holzmann97,DBLP:conf/cav/CimattiCGGPRST02},
 have been developed for many such logics, based on model checking
 techniques \cite{BaKG08}.

 In AI, verification by model checking has become increasingly popular
 in the autonomous agents and multi-agent systems area. There, many
 logics have been proposed that additionally deal with the
 informational and motivational attitudes of agents
 \cite{DBLP:conf/kr/RaoG91,RaoGeorgeffKR92,DBLP:journals/fuin/LinderHM98,WooldridgeBook00,DBLP:journals/ai/CohenL90,ShLL10}.
 Some recent work has been specifically concerned with formalizing
 multi-agent knowledge/belief and their dynamics
 \cite{vanDitmarshDELbook08,DBLP:conf/kr/Herzig14}.  Moreover, various
 Belief-Desire-Intention (BDI) agent programming languages have been
 developed that operationalize these mental attitudes
 \cite{DBLP:conf/maamaw/Rao96,JasonBook07,DBLP:journals/aamas/Dastani08,DBLP:journals/japll/BoerHHM07}.
 Verification is important in this area as agent autonomy makes it
 crucial to be able to guarantee that the system behaves as required
 \cite{DBLP:journals/cacm/FisherDW13}.  Furthermore, one generally
 wants to ensure that the agents' mental states as well as their
 behavior evolve in a way that satisfies certain properties.  Agent
 logics can be used to specify such properties.
 Much of the verification work in this area focuses on the model
 checking of BDI programs.  For instance, \cite{Bordini.etal:AAMAS03}
 shows how to use the \propername{SPIN} model checker
 \cite{DBLP:journals/tse/Holzmann97} to verify properties of
 finite-state \propername{AgentSpeak} programs.
 \cite{DennisFWB12,DBLP:journals/cacm/FisherDW13} compile BDI programs
 and agent properties to verify into \propername{Java} and use
 \propername{JPF} \cite{DBLP:journals/ase/VisserHBPL03} to model check
 them.  \cite{LomuscioQR09} develops \propername{MCMAS}, a symbolic
 model checker specifically for multi-agent systems.
\cite{ADKLM10} develops a theorem proving-based
verification framework for BDI programs that uses a PDL-like
logic. 

In the situation calculus, there is also some previous work on verification.
%
Perhaps the first such work is
\cite{DeTR97}, where verification of possibly non-terminating
\propername{Golog} \cite{Levesque:JLP97-Golog} programs is addressed,
though no effective techniques are given. 
Focusing on the propositional situation calculus (where fluents have
only the situation as argument), \cite{DBLP:conf/ijcai/Ternovskaia99}
presents decidable verification techniques.
In \cite{GuKi06}, these techniques are generalized to a
one-object-argument fluents fragment of the situation calculus, and in  \cite{DBLP:conf/ijcai/GuS07} to theories expressed in two-object-argument fragment. 
Techniques for verification resorting to second-order theorem proving
with no decidability guarantees are presented in
\cite{ShLL10,DBLP:conf/atal/ShapiroLL02}, where the \propername{CASLve}
verification environment for multi-agent \propername{ConGolog}
\cite{DeGiacomoLL:AIJ00-ConGolog} programs is described.
In \cite{Classen:KR08},  \emph{characteristic graphs}  for  programs are introduced to define a form of regression over programs to be used as a pre-image computation step in  (sound) procedures for verifying \propername{Golog} and \propername{ConGolog} programs inspired by model checking.  Verification of programs over a two-variable fragment of the situation calculus is shown to be decidable in \cite{DBLP:conf/aaai/ClassenLLZ14}.
\cite{DBLP:journals/ai/KellyP10} establishes conditions for verifying
loop invariants and persistence properties. 
Finally, \cite{DBLP:conf/kr/GiacomoLP10,DBLP:conf/ijcai/SardinaG09}
propose techniques (with model-checking ingredients) to reason about
infinite executions of \propername{Golog} and \propername{ConGolog}
programs based on second-order logic exploiting fixpoint approximates.

More recently, work closely related to ours 
\cite{DBLP:conf/kr/GiacomoLP12,DBLP:conf/ijcai/GiacomoLP13,DeGiacomoLPVAAMAS14,DBLP:conf/ecai/GiacomoLPV14}
has shown that one gets robust decidability results for temporal
verification of situation calculus action theories under the
assumption that in every situation the number of object tuples forming
the extension of each fluent is bounded by a constant.
In particular, \cite{DBLP:conf/kr/GiacomoLP12} introduced bounded
situation calculus basic action theories;  
\cite{DBLP:conf/kr/GiacomoLP12} however, assumes standard names for
the object domain and, more significantly, disallows quantification
across situations in the verification language. 
In the present paper, which
is a direct extension of \cite{DBLP:conf/kr/GiacomoLP12}, 
both of these limitations are removed.  
In \cite{DBLP:conf/ijcai/GiacomoLP13} an extended language with an
explicit knowledge operator was considered, while in
\cite{DeGiacomoLPVAAMAS14} online executions (i.e., executions where
the agent only performs actions that it knows are executable)
 and progression are
studied;
like \cite{DBLP:conf/kr/GiacomoLP12}, these papers also assume
standard names and rule out
quantification across situations from the verification language.
\cite{DBLP:conf/ecai/GiacomoLPV14} addresses verification over online
executions with sensing in bounded situation calculus theories,
adopting as verification language a first-order variant of Linear
Temporal Logic (\propername{FO-LTL}), again without quantification
across situations.

The work in this paper is also closely related to
\cite{Belardinelli.etal:JAIR14}. There, an ad-hoc formalism for
representing action and change is developed with the purpose of
capturing data-aware artifact-centric processes. This formalism
describes action preconditions and postconditions in first-order
logic, and induces \emph{genericity} \cite{AbHV95} --- there called
\emph{uniformity} --- on the generated transition system. Intuitively
genericity requires that if two states are isomorphic they
induce 
the ``same'' transitions (modulo isomorphism). This means, in
particular, that the system is essentially \emph{Markovian}
\cite{Reiter01-Book}.
As verification language, they consider \propername{FO-CTL}, a
first-order variant of \propername{CTL} that allows for quantifying
across states \emph{without requiring object persistence}, as,
instead, we do here.  Their results imply that one can construct a
finite-state transition system over which the \propername{FO-CTL} formula of
interest can be verified. However, differently from our case, such a
transition system depends also on the number of variables in the
formula.
While also bounded situation calculus action theories enjoy
genericity, it is easy to see that, without assuming object
persistence, we immediately lose the possibility of abstracting to a
finite transition system independently from the formula to
verify. This is true even if we drop completely fixpoints. Indeed,
assume that we have an action replaces an object in the active domain
by one in its parameters. Then, without persistence, for any bound $n$
over the number of objects in a candidate finite abstraction, we can
write a (fixpoint-free) formula saying that there exists a finite run
with more than $n$ distinct objects:
\[\begin{array}{l}
\exists x_1. \inadom(x_1) \land \DIAM{(\exists x_2. \inadom(x_2) \land x_2\neq x_1 \land {}\\
\qquad \DIAM{(\exists x_3 \inadom(x_3) \land x_3\neq x_1 \land x_3\neq x_2 \land {}\\
\qquad \qquad \cdots {}\\
\qquad \qquad \qquad \DIAM{(\exists x_{n+1} \inadom(x_{n+1}) \land x_{n+1}\neq x_1 \land \cdots \land x_{n+1}\neq x_n)})})}
  \end{array}
  \]
  Obviously, this formula is false in the finite abstraction, while
  true in the original transition system, where objects are not
  ``reused''.
  Notice that the formula belongs also to \propername{FO-CTL} and this
  limitation applies to \cite{Belardinelli.etal:JAIR14} as well.
  This observation shows that the persistence condition is crucial to
  get an abstraction that is independent from the formula.

  It is interesting to observe that while dropping persistence is
  certainly a valuable syntactic simplification, the deep reason
  behind it is that generic transition systems, including those
  generated by situation calculus basic action theories, are
  essentially unable to talk about objects that are not in the current
  active domain. If some object that is in the active domain
  disappears from it and reappears again, after some steps, the basic
  action theory will treat it essentially as a fresh object (i.e., an
  object never seen before). Hence, any special treatment of such
  objects must come from the formula we are querying the transition
  system with: for example, we may isolate runs with special
  properties and only on those do verification. The fact that
  \propername{FO-CTL} can drop persistence while maintaining
  decidability of verification over generic transition systems tells
  us that \propername{FO-CTL} is not powerful enough to isolate
  interesting runs to be used as a further assumption for
  verification.

The results in this paper are relevant not only for AI, but also for
other areas of computer science (CS).
There is some work in CS that uses model checking techniques on
infinite-state systems. However, in most of this work the emphasis is
on studying recursive control rather than on a rich data oriented
state description; typically data are either ignored or finitely
abstracted, see e.g., \cite{BCMS01}.
There has recently been some attention paid in the field of business
processes and services to including data into the analysis of
processes \cite{Hull2008:Artifact,GeredeSu:ICSOC2007,Dumas2005:PAIS}.
Interestingly, while we have verification tools that are quite good
for dealing with data and processes separately, when we consider them
together, we get infinite-state transition systems, which resist
classical model checking approaches to verification.  Only lately has
there been some work on developing verification techniques that can
deal with such infinite-state processes
\cite{DHPV:ICDT:09,BPM11,ICSOC11,DBLP:conf/pods/HaririCGDM13,Belardinelli.etal:JAIR14}.
In particular, the form of controlled quantification across situations
in our $\mymu$ language, which requires object persistence in the
active domain, is inspired by the one in
\cite{DBLP:conf/pods/HaririCGDM13}, which in turn extends the
verification logic presented in
\cite{DBLP:conf/kr/GiacomoLP12}. There, the infinite-state data-aware
transition systems (with complete information) to verify are defined
using an ad-hoc formalism based on database operations, and the
decidability results are based on two conditions over the transition
systems, namely run-boundedness and state-boundedness. The latter is
analogous to our situation-boundedness.  In this paper, we make the
idea of boundedness flourish in the general setting offered by the
situation calculus, detailing conditions needed for decidability,
allowing for incomplete information, and exploiting the richness of
the situation calculus for giving sufficient conditions for
boundedness that can easily be used in practice. Such results can find
immediate application in the analysis of data-aware business processes
and services.

%



\section{Conclusion}\label{sec:conclusion}





In this paper, we have defined the notion of \emph{bounded action
  theory} in the situation calculus, where the number of fluent atoms
that hold remains bounded.  We have shown that this restriction is
sufficient to ensure that \emph{verification} of an expressive class
of \emph{temporal properties} remains \emph{decidable}, and is in fact
EXPTIME-complete, despite the fact that we have an infinite domain and
state space.  Our result holds even in the presence of incomplete
information.  We have also argued that this restriction can be adhered
to in practical applications, by identifying interesting classes of
bounded action theories and showing that these can be used to model
typical example dynamic domains.
Decidability is important from a theoretical standpoint, but we stress
also that our result is fully constructive being based on a reduction
to model checking of an (abstract) finite-state transition system. An
interesting future enterprise is to build on such a result to develop
an actual situation calculus verification tool.

A future research direction of particular interest is a
more systematic investigation of specification patterns for obtaining
boundedness.  This includes patterns that provide bounded persistence
and patterns that model bounded/fading memory.  These questions should
be examined in light of different approaches that have been proposed
for modeling knowledge, sensing, and revision in the situation
calculus and related temporal logics
\cite{DBLP:journals/ai/ScherlL03,DBLP:conf/ismis/DemolombeP00,DBLP:journals/ai/ShapiroPLL11,vanDitmarshDELbook08}.
This work has already started. In particular, as mentioned earlier,
the approach of this paper has been extended in
\cite{DeGiacomoLPVAAMAS14,DBLP:conf/ecai/GiacomoLPV14} to allow
verification temporal properties over \emph{online executions} of an
agent, where the agent may acquire new information through sensing as
it executes and only performs actions that are feasible according to
its beliefs.  In that work, the agent's belief state is modeled
meta-theoretically, as an action theory that is progressed as actions
are performed and sensing results are obtained.
In \cite{DBLP:conf/ijcai/GiacomoLP13}, temporal epistemic verification
is tackled within a language-theoretic viewpoint, where the situation
calculus is extended with a knowledge modality
\cite{DBLP:journals/ai/ScherlL03}.  The form of boundedness studied in
that case requires that the number of object tuples that the agent
thinks may belong to any given fluent be bounded.  In
\cite{DeGiacomoLPVAAMAS14,DBLP:conf/ecai/GiacomoLPV14}, instead, it is
only required that number of distinct tuples entelied to belong to a
fluent is bounded, while the number of tuples that are in the
extension of a fluent in some model of the theory need not be
bounded. More work is needed to fully reconcile these meta-theoretic
and language-theoretic approaches.

Finally, an important topic for future work is to tackle verification
of agent programs \cite{DeTR97}, possibly expressed in a situation
calculus-based high-level language like
\Golog~\cite{Levesque:JLP97-Golog} or
\ConGolog~\cite{DeGiacomoLL:AIJ00-ConGolog}.
Some cases where verification of \ConGolog\ programs is decidable are
identified in \cite{DBLP:conf/aaai/ClassenLLZ14}.
It woud be interesting to extend our framework to support such a form
of verification as well.  This is not immediate, as a temporal
property may hold over all executions of a program without holding
over all branches of the situation tree.  To extend our approach to
programs, we need to ensure that not just the agent's beliefs but the
whole program configuration remains bounded.


\section*{Acknowledgements}
The authors acknowledge the support of: 
Ripartizione Diritto allo Studio, Universit\`a e Ricerca Scientifica of Provincia Autonoma di Bolzano--Alto Adige, under project VeriSynCoPateD (\emph{Verification and
 Synthesis from Components of Processes that Manipulate Data});  
 EU Commission, under the IP project n.~FP7-318338 Optique (\emph{Scalable End-user
 Access to Big Data});
and the National Science and Engineering Research Council of Canada.


\bibliographystyle{splncs03}
\bibliography{aij}

\end{document}